\documentclass{article}

\usepackage[preprint]{neurips_2026}

\usepackage[utf8]{inputenc}
\usepackage[T1]{fontenc}
\usepackage{hyperref,url,booktabs}
\usepackage{amsfonts,amsmath,amssymb,amsthm}
\usepackage{nicefrac,microtype}
\usepackage{xcolor}
\usepackage[most,skins]{tcolorbox}
\usepackage{graphicx,wrapfig}
\usepackage{adjustbox}
\usepackage{bm}
\usepackage{mdframed}
\usepackage{nameref}
\usepackage{pifont}
\usepackage{tabularx}
\usepackage{threeparttable}
\usepackage{placeins}

\usepackage[normalem]{ulem}
\usepackage{todonotes}

\hypersetup{colorlinks=true,linkcolor=red!85!black,citecolor=purple,urlcolor=cyan}
\setcitestyle{round}

\newtheorem{lemma}{Lemma}

\newtheorem{claim}{Claim}
\newtheorem{example}{Example}

\surroundwithmdframed[linewidth=0.8pt, innertopmargin=8pt, innerbottommargin=8pt]{mainproblem}

\def\R{\mathbb{R}} 
\def\E{\mathbb{E}}
\def\N{\mathcal{N}}
\def\x{{\textcolor{blue!70!green}{\mathbf{x}}}}
\def\y{\mathbf{y}}
\def\v{\mathbf{v}}
\def\u{\mathbf{u}}

\def\z{\mathbf{z}}
\newcommand{\stderr}[1]{{\scriptsize\textcolor{gray}{$\,\pm #1$}}}
\newcommand{\calF}{\mathcal{F}}
\newcommand{\calP}{\mathcal{P}}

\newcommand{\calD}{\mathcal{D}}
\newcommand{\calL}{\mathcal{L}}
\newcommand{\calR}{\mathcal{R}}
\newcommand{\calT}{\mathcal{T}}

\newcommand{\dFdrt}{\frac{\delta \calF}{\delta \rho_t}}
\newcommand{\dFdrtt}{\frac{\delta \calF^{\theta}}{\delta \rho_t^{\theta}}}

\DeclareMathOperator*{\argmin}{arg\,min}

\newcommand{\supp}{\operatorname{supp}}
\newcommand{\pot}{V}
\newcommand{\pote}{U}
\newcommand{\inter}{W}
\newcommand{\KL}{\operatorname{KL}}
\newcommand{\Div}{\operatorname{div}_\x}
\newcommand{\kernel}{\phi}
\newcommand{\wt}{w^{\theta}}

\title{Wasserstein Residuals: Learning Gradient Flows from Population Dynamics }

\author{%
  \textbf{Markus Heinonen} $^{1,2}$\thanks{Correspondence to \texttt{markus@basis.ai}} \quad
  \textbf{Yair Shenfeld} $^{1,3}$ \quad
  \textbf{Ricardo Baptista} $^{1,4}$ \\
  \textbf{Daniel Waxman} $^{1,5}$ \quad
  \textbf{Dmitry Batenkov} $^{1}$ \quad
  \textbf{Tim Cooijmans} $^{1}$ \quad
  \textbf{Eli Bingham} $^{1}$ \\[6pt]
  \normalfont
  $^{1}$ Basis Research Institute \quad
  $^{2}$ Aalto University \quad
  $^{3}$ Brown University \\
  $^{4}$ University of Toronto \quad
  $^{5}$ MIT \quad
}

\begin{document}

\maketitle 

\begin{abstract}
Reconstructing population dynamics is a central problem in the physical and data sciences. Often, the dynamics are modeled as a Wasserstein gradient flow (WGF): a curve of distributions driven by an energy functional. Though there are multiple mathematical characterizations of a WGF, the dominant algorithmic approach relies on the Jordan--Kinderlehrer--Otto (JKO) scheme. JKO-based methods are inflexible to time discretisation and require solving costly optimal transport problems. We take a \emph{residual} approach, enforcing the continuity equations via a non-negative loss function whose minimum is the WGF. Combined with a data-fitting divergence, this gives a single global objective. This perspective unifies several existing methods and leads to a new particle-based method, \emph{stitching}, that is simulation-free and robust to large gaps between observations. We demonstrate that the stitching method achieves state-of-the-art performance across trajectory inference benchmarks. For code see \url{github.com/BasisResearch/wasserstein-residuals}. 
\end{abstract}

\section{Introduction}
\label{sec:intro}

Reconstructing how a population evolves from a handful of snapshots is a central challenge across many scientific fields, ranging from computational biology \citep{schiebinger2019optimal} to crowd dynamics \citep{maury2010macroscopic}. A common language for such problems is that of \emph{Wasserstein gradient flows}: curves of probability measures $t\mapsto \rho_t$ driven by the steepest descent of an energy functional $\calF$ \citep{ambrosio2005gradient,santambrogio2015optimal}. We study the inverse problem: recovering the energy $\calF$ from observed snapshots so that its gradient flow $\rho:=(\rho_t)$ fits the snapshots $\{q_t\}$ for a collection of observation times $t\in\calT_{\mathrm{obs}}$. 


The dominant line of work \citep{bunne2022proximal,terpin2024learning,persiianov2025ijkonet} on this problem learns the functional $\calF$ by relying on the Jordan--Kinderlehrer--Otto (JKO) definition of Wasserstein gradient flows \citep{jordan1998variational}. This requires solving multiple optimal transport (OT) problems, which can be costly, and lead to inaccuracies under long temporal gaps; see \autoref{fig:wavy_valley}. 

We take a different perspective: a residual loss that vanishes exactly when $\rho$ is a gradient flow of $\calF$.
Combined with a data-fitting divergence at observed times, the objective directly enforces both the gradient flow constraint and the data fit. The perspective unifies existing paradigms --- Path-Finding \citep{liu2026generative} and Action Matching \citep{neklyudov2023action} --- as instantiations of one residual framework. The residual framework leads us to a new particle-based method which achieves state-of-the-art results on real-world datasets. 

\paragraph{Contributions.} We summarize our contributions along the following three axes:
\begin{itemize}\itemsep0.2em
\item We recast the problem as minimizing a nonnegative residual term, a framework subsuming Path-Finding \citep{liu2026generative} and Action Matching \citep{neklyudov2023action}.

\item We introduce \emph{stitching}, a simulation-free KDE-based method, which promotes the curve $\rho$ to a first-class learnable variable alongside $\calF$ and tolerates large gaps between snapshots.
\item We achieve state-of-the-art performance on single-cell RNA trajectory inference, and interaction dynamics recovery.
\end{itemize}

\paragraph{Outline.} In \autoref{sec:problem} we present background for learning WGF from population dynamics. \autoref{sec:residuals} presents the general Wasserstein residuals framework, and \autoref{sec:vel_res} introduces our stitching method. \autoref{sec:experiments} describes our numerical experiments.
\begin{figure*}[tb]
\centering
\includegraphics[width=\linewidth]{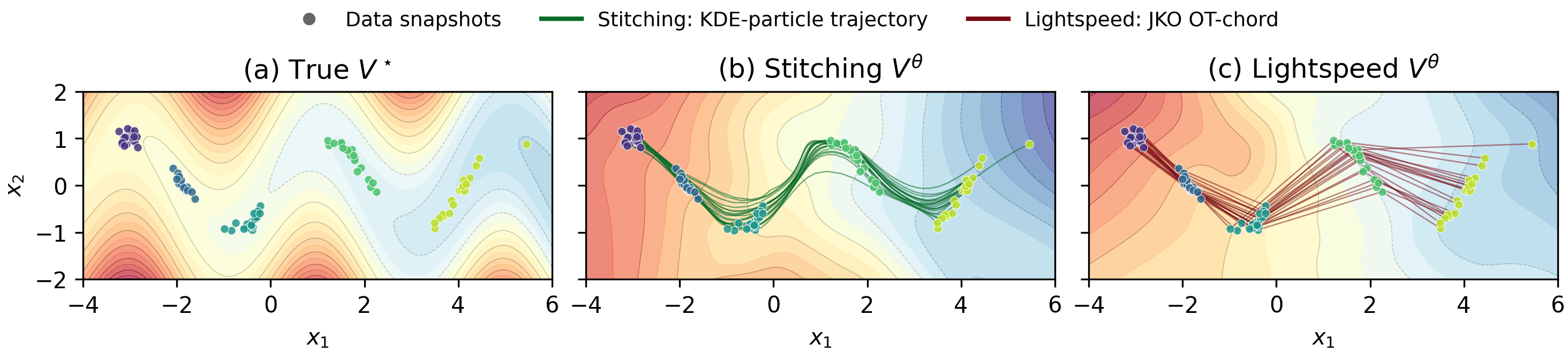}
\caption{\textbf{Stitching recovers curvature that JKO chord predictors miss.} Sparse SDE snapshots ($N=20$) on a sinusoidal valley at $t\in\{0,5,10,20,30\}$. \textbf{(a)} True potential. \textbf{(b)} Stitching's trajectory tracks the valley. \textbf{(c)} Lightspeed \citep{terpin2024learning} does not capture the curved trajectories of the particles; see \autoref{sec:exp_wavy_valley}.}
\label{fig:wavy_valley}
\end{figure*}
\section{Background}
\label{sec:problem}

We work in the Wasserstein space $(\calP_2(\R^D), W_2)$ of probability measures $\calP_2(\R^D)$ over real space $\R^D$ with finite second moment, equipped with the 2-Wasserstein distance $W_2$ \citep{ambrosio2005gradient,villani2021topics,santambrogio2015optimal}; see \autoref{app:wgf_background} for definitions. An absolutely continuous curve $\rho := (\rho_t)_{t\in [0,T]}$ admits a unique minimal velocity field $\v$ satisfying the continuity equation
\begin{equation}
\label{eq:cont_eq_intro}
\partial_t\rho_t(\x) = - \Div\big(\rho_t(\x)\v_t(\x)\big),
\end{equation}
\cite[Thm 8.3.1]{santambrogio2015optimal}. The equivalent ``Lagrangian'' description is via particle trajectories 
\begin{equation}
\label{eq:lagrang_intro}
\dot\x_t = \v_t(\x_t), \quad \x_0\sim\rho_0,
\end{equation}
satisfying $\x_t\sim\rho_t$ for all $t\in[0,T]$.

A \textbf{Wasserstein gradient flow} of a functional $\calF:\calP_2(\R^D)\to\R$ is a curve $\rho$ with velocity
\begin{equation}
\label{eq:grad_flow_intro}
\v_t(\x) = -\nabla_\x\dFdrt(\x),
\end{equation}
where $\dFdrt$ is the first variation of $\calF$ at $\rho_t$. We focus on functionals of the form 
\begin{equation}
\label{eq:functional}
\calF[\rho_t] = \int_{\R^D} \pot(\x)\,\mathrm{d}\rho_t(\x) + \int_{\R^D}\log\rho_t(\x)\,\mathrm{d}\rho_t(\x) + \int_{\R^D\times\R^D}\inter(\x'-\x)\,\mathrm{d}\rho_t(\x)\,\mathrm{d}\rho_t(\x'),
\end{equation}
where $\pot:\R^D\to \R$ is a \emph{potential} term, and $\inter:\R^D\to \R$ is a symmetric $\inter(-\x)=\inter(\x)$ \emph{interaction} term \cite[\S 5.2.2]{villani2021topics}. Besides their ubiquity, functionals of the form \eqref{eq:functional} have the advantage of having an analytic Wasserstein gradient
\begin{equation}
\label{eq:grad_functional}
\nabla_\x\dFdrt(\x) = \nabla_\x\pot(\x) + \nabla_\x\log\rho_t(\x) + \nabla_\x(\inter\ast\rho_t)(\x).
\end{equation}
A few standard examples of Wasserstein functionals appear in \autoref{app:wgf_background}. 

Let us now state the focus of our work:
\begin{tcolorbox}[enhanced,%
    colback=white,%
    colframe=black!40!white,%
    colbacktitle=white,%
    coltitle=black,%
    fonttitle=\bfseries, %
    title=Goal,%
    attach boxed title to top left={yshift=-3.0mm, xshift=0.6cm},
    boxed title style={
        frame hidden,       
        boxrule=0pt        
    }]
    Given a collection of discrete snapshots at times $\calT_{\mathrm{obs}}\subseteq [0,T]$ --- where, for simplicity, we assume $0\in \calT_{\mathrm{obs}}$ ---  and samples from marginal distributions $\{q_t\}_{t\in \calT_{\mathrm{obs}}}$, find $\calF$ whose gradient flow $\rho$ satisfies $\rho_t = q_t$ for all $t\in \calT_{\mathrm{obs}}$.
\end{tcolorbox}




\section{Wasserstein residuals}
\label{sec:residuals}
Following the classical reference \citet[Ch.~11]{ambrosio2005gradient}, gradient flows in metric spaces admit four equivalent formulations: a \emph{Tangent} condition (the continuity equation), an \emph{Energy Dissipation Equality} (EDE), the \emph{JKO scheme}, and an \emph{Evolution Variational Inequality} (EVI).  In  \autoref{app:wgf_background} we provide 
the Euclidean intuition for this four-way equivalence. For the purpose of numerical computations with Wasserstein gradient flows, the EVI is an \emph{in}equality and is thus not amenable to residual minimization. Whereas the JKO scheme has been widely used \citep{bunne2022proximal,terpin2024learning,persiianov2025ijkonet},
we focus on the Tangent and EDE formulations which have been largely overlooked. 

Both the Tangent and EDE formulations are based on \eqref{eq:cont_eq_intro} with velocity $\v_t = -\nabla\dFdrt$. 
The tangent formulation requires
\begin{equation}
\label{eq:density_form}\tag{density constraint}
  \partial_t\rho_t(\x) = \Div\left(\rho_t(\x)\nabla_\x\dFdrt(\x)\right),
\end{equation}
while the EDE is formulated in terms of the velocity field $\v_t$,
\begin{equation}
\label{eq:velocity_form}\tag{velocity constraint}
  \v_t(\x) = -\nabla_\x\dFdrt(\x).
\end{equation}
Each form is a pointwise constraint that holds if and only if $\rho$ is a gradient flow of $\calF$. 

\subsection{Density and velocity residuals}
\label{subsec:two_residuals}
The \eqref{eq:density_form} and \eqref{eq:velocity_form} produce corresponding residuals:
\begin{align}
  \calR_{\mathrm{dens}}[\calF, \rho] &:= \int_0^T\int_{\R^D}\left\|\partial_t\rho_t(\x)-\Div\left(\rho_t(\x)\nabla_\x\dFdrt(\x)\right)\right\|^2\rho_t(\x)\,\mathrm{d}\x\,\mathrm{d}t \tag{density residual}\label{eq:density_res} \\
  \calR_{\mathrm{vel}}[\calF, (\rho, \v)] &:= \int_0^T\int_{\R^D}\left\|\v_t(\x)+\nabla_\x\dFdrt(\x)\right\|^2\rho_t(\x)\,\mathrm{d}\x\,\mathrm{d}t . \tag{velocity residual}\label{eq:velocity_res}
\end{align}
Both residuals equal zero if and only if $\rho$ is a Wasserstein gradient flow of $\calF$. However, the density residual requires higher-order differentiation (due to the $\Div$ operator), so instead we focus on the velocity residual which avoids this issue. 

\subsection{Coupling residuals to data}
\label{subsec:data_coupling}

The \eqref{eq:velocity_res} enforces that $\rho$ is a gradient flow of $\calF$, but it does not couple $\rho$ to the data $\{q_t\}_{t\in \calT_{\mathrm{obs}}}$. To this end we add a statistical divergence $\calD(\rho_t,q_t)$ at observed times, where $\calD$ satisfies $\calD(\rho_t,q_t)=0\Rightarrow \rho_t=q_t$. This yields the global objective
\begin{equation}
\label{eq:dens_res_data}
  \calL(\calF,\rho,\v) = \lambda\calR_{\mathrm{vel}}[\calF,\rho,\v]+\sum_{t\in \calT_{\mathrm{obs}}}\calD(\rho_t,q_t), 
\end{equation}
with $\lambda>0$. Standard choices for $\calD$ include the Kullback--Leibler divergence (i.e., likelihood maximization), Fisher divergence (i.e., score matching (SM)~\citep{hyvarinen2005estimation}), and denoising score matching (DSM) \citep{vincent2011connection}; see \autoref{app:divergences}.

\section{Stitching}
\label{sec:vel_res}
\label{subsec:stitching}
We assume that the true functional $\calF$ is of the form \eqref{eq:functional}, and parametrize an approximating functional $\calF^\theta$ with neural networks for the potential $\pot^\theta$ and the interaction kernel $\inter^\theta$, so that the Wasserstein gradient can be written as
\begin{equation}
\label{eq:grad_functional_nn}
\nabla_\x\dFdrtt(\x) = \nabla_\x\pot^{\theta}(\x) + \nabla_\x\log\rho_t^{\theta}(\x) + \nabla_\x(\inter\ast\rho_t^{\theta})(\x).
\end{equation}

\paragraph{Curve parametrization.} 
The parametrization of the curve $\rho^{\theta}$ is done by a kernel density estimate (KDE) over moving differentiable trajectories $[0,T]\ni t\mapsto \x_{t,k}^\theta$,
\begin{equation}\label{eq:particle_param}
\rho_t^{\theta}(\x) := \sum_{k=1}^{N} \wt_k\,\kernel(\x - \x_{t,k}^\theta),
\qquad \wt_k\ge 0,\quad \sum_{k=1}^N \wt_k=1,
\end{equation}
with $\kernel$ a smooth probability kernel. Given the parametrization \eqref{eq:particle_param}, the associated velocity field is (cf. Claim \ref{cl:vec_KDE} in \autoref{app:stitching_derivation})
\begin{equation}
\label{eq:kde_vel}
\v_t^{\theta}(\x)=\frac{\sum_{k=1}^{N} \wt_k \,\kernel(\x - \x_{t,k}^\theta) \,\dot{\x}_{t,k}^\theta}{\sum_{l=1}^{N} \wt_l\,\kernel(\x - \x_{t,l}^\theta) },
\end{equation}
which is defined in terms of the particles $\{\x_{t,k}^\theta\}$ and their velocities $\{\dot{\x}_{t,k}^\theta\}$. 

\paragraph{Continuous objective.} 
With this KDE parametrization, the objective \eqref{eq:dens_res_data} reads
\begin{equation}
\label{eq:stitch_ideal_global}
\calL_{\mathrm{stitch}}(\theta) = \int_0^T \int_{\R^D}
\left\|\v_t^\theta(\x)+\nabla_\x\dFdrtt(\x)\right\|^2 \rho_t^{\theta}(\x)\mathrm{d}\x\,\mathrm{d}t+\sum_{t\in\calT_{\mathrm{obs}}}\calD(\rho_t^\theta,q_t),
\end{equation}
with $\v_t^\theta$ as in \eqref{eq:kde_vel} and $\nabla_\x\dFdrtt$ as in \eqref{eq:grad_functional_nn}.


\paragraph{Empirical-measure approximation.}
To make the computation of $\calL_{\mathrm{stitch}}$ more efficient we make the approximation
\begin{equation}
\label{eq:kde_approx}
\rho_t^{\theta}(\x) = \sum_{k=1}^{N} \wt_k\,\kernel(\x - \x_{t,k}^\theta)\approx \sum_{k=1}^{N} \wt_k\,\delta_{\x_{t,k}^\theta}(\x),
\end{equation}
i.e., we approximate the KDE by the empirical measure on its centers. We use the approximation \eqref{eq:kde_approx} both when integrating in the velocity-residual term against $\rho_t^{\theta}(\x)\mathrm{d}\x$, and when evaluating the velocity $\v_t^{\theta}$ of \eqref{eq:kde_vel}. Specifically, the velocity at a particle simplifies to $\v_t^{\theta}(\x_{t,k}^\theta) = \dot{\x}_{t,k}^\theta$ (cf. Claim \ref{cl:cont_eq_particle} in \autoref{app:stitching_derivation}), which leads to
\begin{equation}
\label{eq:stitch_ideal_global_kde}
\calL_{\mathrm{stitch}}(\theta) \approx \int_0^T \sum_{k=1}^N \wt_k
\left\|\dot{\x}_{t,k}^\theta+\nabla_\x\dFdrtt(\x_{t,k}^\theta)\right\|^2 \mathrm{d}t+\sum_{t\in\calT_{\mathrm{obs}}}\calD(\rho_t^\theta,q_t).
\end{equation}

\paragraph{Time discretization.}
We discretize time $0=t_0<t_1<\cdots<t_{K-1}=T$ and approximate $\dot{\x}_{t_j,k}^\theta \approx \Delta \x_{t_{j+1}, k}^\theta / \Delta t_{j+1}$, where $\Delta t_{j+1} := t_{j+1} - t_j$ is the width of the $j$th time interval and $\Delta \x_{t_{j+1}, k}^\theta$ is a discretization rule, e.g.,  forward Euler $ \Delta \x_{t_{j+1}, k}^\theta:= \x_{t_{j+1},k}^\theta - \x_{t_j,k}^\theta$.

\begin{tcolorbox}[enhanced,%
    colback=white,%
    colframe=black!40!white,%
    colbacktitle=white,%
    coltitle=black,%
    fonttitle=\bfseries, %
    title=Stitching Objective,%
    attach boxed title to top left={yshift=-3.5mm, xshift=0.6cm},
    boxed title style={
        frame hidden, 
        boxrule=0pt 
    }]
\begin{equation}
\label{eq:stitch_ideal_global_approx}
\widehat\calL_{\mathrm{stitch}}(\theta) := 
\sum_{j=0}^{K-2}\sum_{k=1}^N \wt_k
\frac{1}{\Delta t_{j+1}}\left\|\Delta \x_{t_{j+1}, k}^\theta+\Delta t_{j+1}\nabla_\x\dFdrtt(\x_{t_j,k}^\theta)\right\|^2+\sum_{t\in\calT_{\mathrm{obs}}}\calD(\rho_t^\theta,q_t).
\end{equation}
\end{tcolorbox}
\paragraph{Data divergence and functional.}
For the data divergence term $\calD(\rho_t^\theta,q_t)$ we still use the KDE parametrization \eqref{eq:particle_param}, which in turn allows us to use any of the divergences in \autoref{app:divergences}. For Equation \eqref{eq:grad_functional_nn}, the term $\nabla\log\rho_t^\theta$ uses \eqref{eq:particle_param} (see \autoref{app:stitching_derivation}), while the interaction term is again approximated using the centers:
\begin{equation}
\label{eq:approx_inter_kde}
\nabla(\inter^\theta\ast\rho_t^{\theta})(\x_{t_j,k}^\theta)\approx
 \sum_{l=1}^N \wt_l\,\nabla\inter^\theta(\x_{t_j,k}^\theta - \x_{t_j,l}^\theta).
\end{equation}
\paragraph{Relation to prior velocity-residual instantiations.}
The \eqref{eq:velocity_res} has previously been instantiated as a normalizing flow \citep{liu2026generative}, and can be seen as an instance of Action Matching \citep{neklyudov2023action}; see \autoref{app:related_work}. Stitching's distinction is that the trajectory $\rho^\theta$ is itself a learnable particle cloud --- rather than the output of a learned flow that requires ODE integration.

\section{Experiments}
\label{sec:experiments}

\subsection{Illustrative example: continuous flow vs.\ JKO chord}
\label{sec:exp_wavy_valley}
\autoref{fig:wavy_valley} on page~\pageref{fig:wavy_valley} contrasts continuous flow against JKO chord interpolation on the \emph{wavy valley} potential $V(x_1,x_2) = 0.6(x_2 - \sin(\pi x_1/2))^2 - 0.3\,x_1$: particles flow along a sinusoidal floor under an SDE $\mathrm{d}x = -\nabla V\,\mathrm{d}t + \sqrt{2\beta}\,\mathrm{d}W$ with $\beta = 0.00625$. We observe $N=20$ particles at five irregular times $t \in \{0, 5, 10, 20, 30\}$; the $t=10\to 20$ gap is too wide for a first-order JKO predictor to interpolate the curve faithfully. Stitching and JKOnet$^\star$ \citep{terpin2024learning} jointly learn $V^\theta$ and the diffusion coefficient (full setup in \autoref{app:wavy_valley}). 

The purpose of the experiment is to highlight the flexibility of our explicitly parametrized trajectory curves. Whereas JKOnet$^\star$'s first-order chord interpolation is restricted to straight-line segments between consecutive snapshots, the trajectories found by our stitching algorithm correctly curve according to the wavy valley. As a consequence, we more accurately recover the underlying potential,
with $R^2=0.62$ vs.\ $0.51$, where $R^2$ is the coefficient of determination between learned and ground-truth potentials evaluated on the data support.

\subsection{Synthetic potential recovery}
\label{sec:exp_synthetic}
We benchmark on the $15$ two-dimensional potentials of \citet{terpin2024learning}
$N_{\text{sim}}=2{,}000$ ODE particles are integrated for $T=5$ steps of $\Delta t = 0.01$, and split $50/50$ into train / held-out test snapshots. Under the \emph{paired} setup we observe the particle trajectories, while in the \emph{unpaired} setup, snapshots are decorrelated (See \autoref{fig:flowers_paired_vs_unpaired}). All methods use the same $(64,64)$ MLP architecture for $V^\theta$; stitching uses $N=1{,}000$ particle trajectories of length $K=50$. We report the metrics of \citet{persiianov2025ijkonet}: (i) the EMD measuring forward-prediction accuracy; (ii) the $L^2$-UVP measuring recovery of potential gradient; and (iii) the $\mathrm{Bd}^2_{W_2}$-UVP measuring distributional moment match (See \autoref{app:synthetic_full}).

\begin{table}[tb]
\centering
\caption{Comparison on the $6$ potentials of \citet[Tab.~3]{persiianov2025ijkonet} most sensitive to the paired$\to$unpaired transition. Three methods (J = JKOnet$^\star_V$, I = iJKOnet$_V$, S = stitching) on paired and unpaired regimes. EMD measures forward-prediction accuracy; $L^2$-UVP measures recovery of the underlying potential's gradient; $\mathrm{Bd}^2_{W_2}$-UVP measures distributional moment match. Stitching's potential-recovery and moment-match errors are effectively unchanged when snapshots are decorrelated, while JKO-based methods degrade or collapse. JKO methods exploit paired snapshots best on forward prediction (EMD).}
\label{tab:synthetic2d_abs}
\scriptsize
\setlength{\tabcolsep}{3pt}
\adjustbox{max width=\linewidth}{%
\begin{tabular}{llrrrrrrrrrrrrrrrrrr}
\toprule
 & & \multicolumn{9}{c}{paired} & \multicolumn{9}{c}{unpaired} \\
\cmidrule(lr){3-11} \cmidrule(lr){12-20}
 & & \multicolumn{3}{c}{EMD\,$\downarrow$} & \multicolumn{3}{c}{$L^2$-UVP\,$\downarrow$} & \multicolumn{3}{c}{$\mathrm{Bd}^2_{W_2}$-UVP\,$\downarrow$} & \multicolumn{3}{c}{EMD\,$\downarrow$} & \multicolumn{3}{c}{$L^2$-UVP\,$\downarrow$} & \multicolumn{3}{c}{$\mathrm{Bd}^2_{W_2}$-UVP\,$\downarrow$} \\
\cmidrule(lr){3-5} \cmidrule(lr){6-8} \cmidrule(lr){9-11} \cmidrule(lr){12-14} \cmidrule(lr){15-17} \cmidrule(lr){18-20}
\# & potential & J & I & S & J & I & S & J & I & S & J & I & S & J & I & S & J & I & S \\
\midrule
 1 & flowers       & $\mathbf{0.01}$  & $0.01$           & $0.30$           & $149$  & $-$ & $\mathbf{0.00}$  & $\mathbf{0.000}$ & $\mathbf{0.000}$ & $0.18$           & $0.39$ & $\mathbf{0.30}$  & $0.31$           & $151$  & $-$ & $\mathbf{0.01}$ & $1.7$ & $0.35$            & $\mathbf{0.24}$ \\
 4 & zigzag\_ridge & $0.04$           & $\mathbf{0.02}$  & $0.29$           & $4041$ & $-$ & $\mathbf{0.03}$  & $0.05$           & $\mathbf{0.003}$ & $0.17$           & $0.38$ & $\mathbf{0.29}$  & $0.30$           & $3870$ & $-$ & $\mathbf{0.04}$ & $1.7$ & $0.37$            & $\mathbf{0.29}$ \\
 6 & watershed     & $\mathbf{0.002}$ & $0.006$          & $0.30$           & $87$   & $-$ & $\mathbf{0.00}$  & $\mathbf{0.000}$ & $\mathbf{0.000}$ & $0.19$           & $0.38$ & $\mathbf{0.30}$  & $0.30$           & $89$   & $-$ & $\mathbf{0.01}$ & $1.5$ & $0.35$            & $\mathbf{0.24}$ \\
 7 & ishigami      & $\mathbf{0.01}$  & $0.02$           & $0.30$           & $719$  & $-$ & $\mathbf{0.01}$  & $\mathbf{0.000}$ & $0.001$          & $0.17$           & $0.38$ & $0.30$           & $\mathbf{0.30}$  & $733$  & $-$ & $\mathbf{0.01}$ & $1.6$ & $0.35$            & $\mathbf{0.25}$ \\
 8 & friedman      & $0.09$           & $\mathbf{0.06}$  & $0.29$           & $3899$ & $-$ & $\mathbf{0.06}$  & $0.02$           & $\mathbf{0.002}$ & $0.20$           & $0.40$ & $\mathbf{0.29}$  & $0.31$           & $3747$ & $-$ & $\mathbf{0.06}$ & $1.8$ & $0.33$            & $\mathbf{0.27}$ \\
11 & wavy\_plateau & $0.53$           & $\mathbf{0.14}$  & $0.27$           & $3411$ & $-$ & $\mathbf{0.02}$  & $7.4$            & $\mathbf{0.40}$  & $0.51$           & $0.48$ & $\mathbf{0.28}$  & $0.29$           & $3458$ & $-$ & $\mathbf{0.04}$ & $4.9$ & $0.63$            & $\mathbf{0.54}$ \\
\bottomrule
\end{tabular}%
}
\end{table}
\begin{figure}[tb]
\centering
\includegraphics[width=\linewidth]{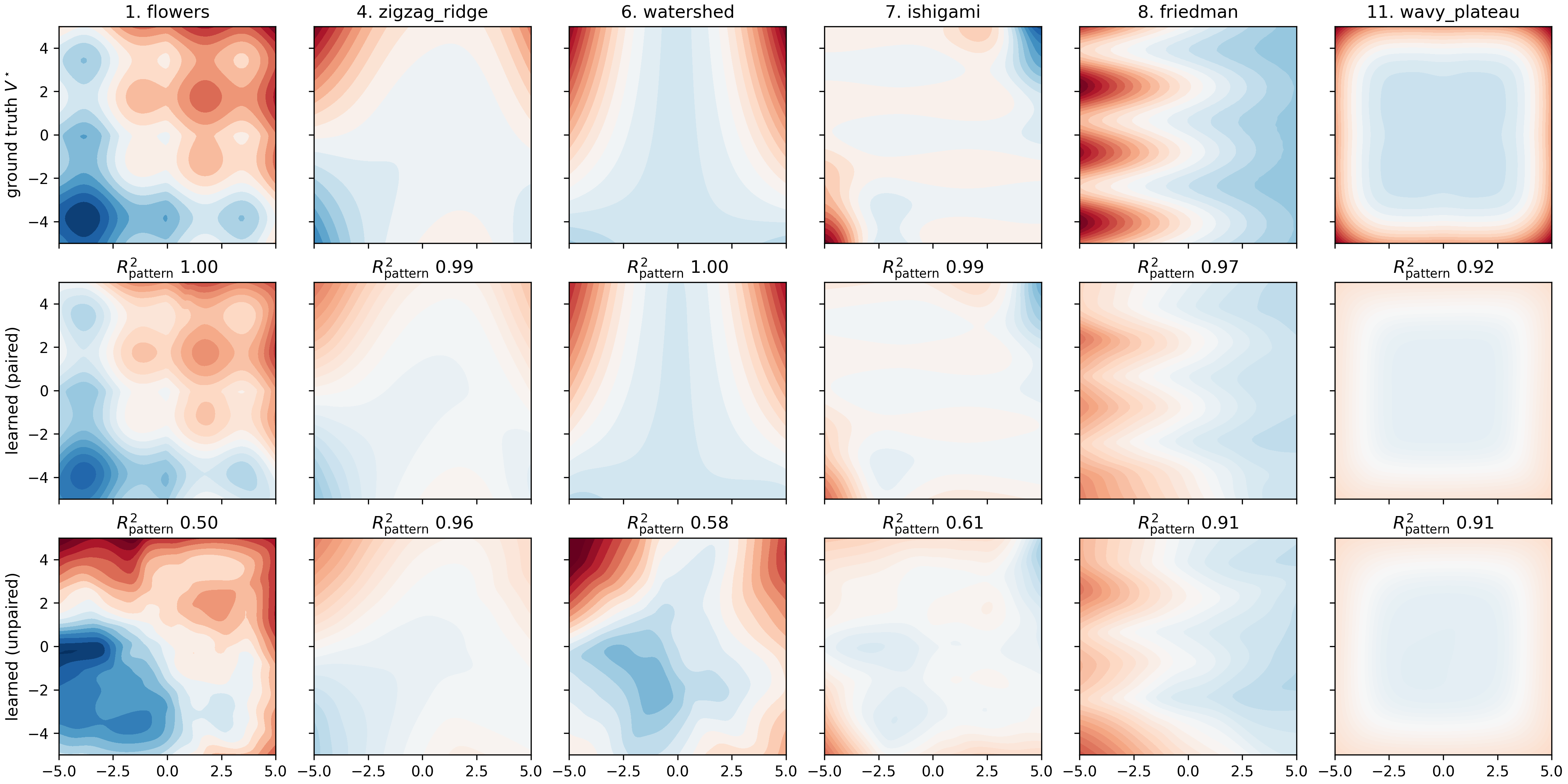}
\caption{\textbf{Stitching's potential recovery is qualitatively unchanged when consecutive snapshots are decorrelated.} Top: ground-truth $V$. Middle / bottom: stitching's $V^\theta$ trained on paired / unpaired snapshots. Per-panel labels: $R^2_{\mathrm{pattern}}$. Full $15$-potential galleries in \autoref{app:synthetic_full}.}
\label{fig:synthetic_compact}
\end{figure}
Stitching recovers $V$ on $11$ of $14$ informative landscapes in both regimes; failures are the angular \texttt{rotational} and the degenerate \texttt{flat} (full galleries in \autoref{fig:synthetic_gallery_paired} and \autoref{fig:synthetic_gallery_unpaired}).

\autoref{tab:synthetic2d_abs} reports absolute numbers on the $6$ potentials \citet[Tab.~3]{persiianov2025ijkonet} flagged as most sensitive to the paired$\to$unpaired transition. Stitching is consistent across both regimes: paired and unpaired errors stay within $\sim$2$\times$ on every metric and every potential, and stitching never collapses. JKO-step methods, by contrast, exploit paired snapshots and beat stitching on forward prediction (EMD) in the paired regime, but degrade or collapse on the unpaired UVPs --- JKOnet$^\star_V$ in particular fails to recover an informative gradient field at all ($L^2$-UVP $>100\%$) at this default training budget. The underlying reason is that stitching's velocity residual is evaluated against a learnable trajectory rather than an OT coupling between consecutive snapshots, so its quality does not depend on snapshot-to-snapshot pairing.

\subsection{Single-cell trajectory inference}
\label{sec:exp_rna}
We apply stitching to the embryoid body (EB) single-cell RNA sequencing dataset of \citet{moon2019visualizing}, which captures a population of human embryonic stem cells over $27$ days of differentiation as five snapshots at days $\{1\text{--}3,\,6\text{--}9,\,12\text{--}15,\,18\text{--}21,\,24\text{--}27\}$, indexed $t\in\{0,1,2,3,4\}$. We follow the preprocessing of \citet{tong2020trajectorynet} and reduce each cell to its first $5$ principal components, the same setup used by \citet{terpin2024learning} and \citet{persiianov2025ijkonet}. Recovering an underlying energy landscape from population snapshots in single-cell biology is a recurring problem \citep{huizing2025stories}; unlike JKO methods that require well-estimated marginals at every proximal step, stitching optimizes the full trajectory globally and handles long temporal gaps naturally.

We parametrize $\calF^\theta[\rho_t^\theta] = c_V\,\mathbb{E}_{\rho_t^\theta}[V^\theta(\x)] + c_H\,\mathbb{E}_{\rho_t^\theta}[\log\rho_t^\theta]$ with a $(64,64)$ MLP for $V^\theta(\x)$, matching the architecture of \citet{persiianov2025ijkonet}. The time-varying variant $\smash{V^\theta(\x,t)}$ concatenates $t$ to the input. The curve $\rho^\theta$ uses $N=100$ particle trajectories of length $K=50$. Full details in \autoref{app:rna_details}.

\begin{figure}[b]
\centering
\includegraphics[width=\textwidth]{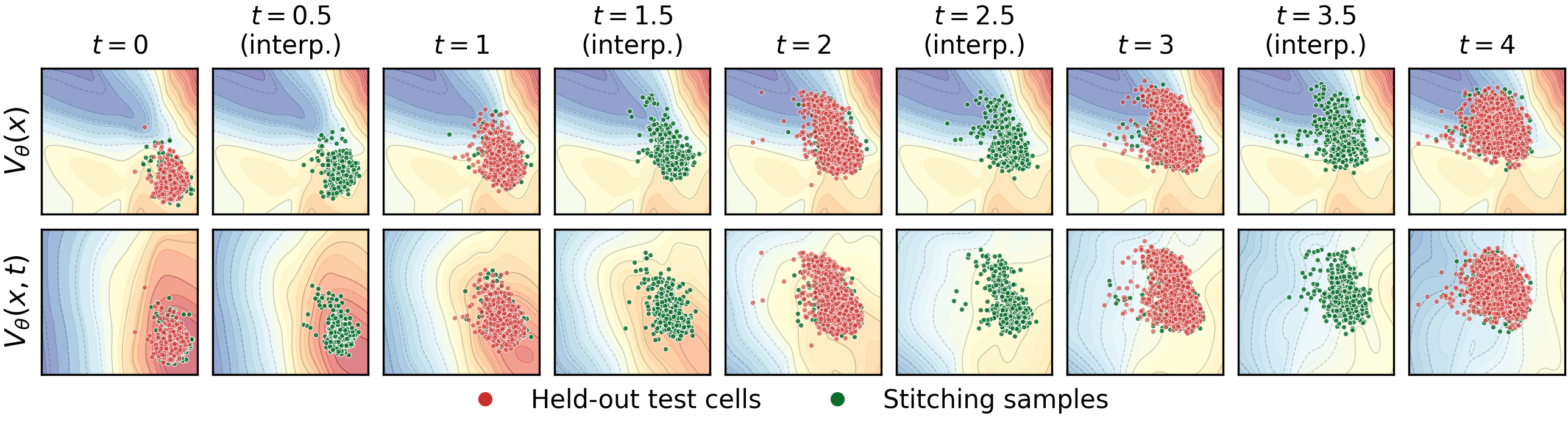}
\caption{\textbf{Both static and time-varying $V^\theta$ drive the population evolution accurately.} Contours of $V^\theta$ in (PC1, PC2) with PC3--5 fixed at the global mean; half-integer columns are unseen at training.}
\label{fig:rna_panels}
\end{figure}

The standard EB benchmark trains on all five observed snapshots and reports $W_1$ between predicted and observed marginals at each transition $\rho_k\to\rho_{k+1}$. \autoref{tab:rna_full} compares stitching against the baselines collected in \citet[Table~5]{persiianov2025ijkonet}; stitching achieves the best performance. \autoref{fig:rna_panels} visualises the time-varying $V^\theta(\x,t)$: its basin migrates with the differentiating cell population, while the static $V^\theta(\x)$ is forced to compromise across the trajectory.

\begin{table}[tb]
\centering
\caption{Full-data EB single-cell benchmark (5D), $W_1$ ($\downarrow$). Baseline results from \citet[Table~5, no standard deviations reported]{persiianov2025ijkonet}. Our results averaged over 5 seeds.}
\label{tab:rna_full}
\adjustbox{max width=\linewidth}{%
\begin{tabular}{lccccll}
\toprule
Method & $t=1$ & $t=2$ & $t=3$ & $t=4$ & Mean & Citation \\
\midrule
Neural SDE & $0.69$ & $0.91$ & $0.85$ & $0.81$ & $0.82$ & \citet{li2020scalable} \\
TrajectoryNet & $0.73$ & $1.06$ & $0.90$ & $1.01$ & $0.93$ & \citet{tong2020trajectorynet} \\
SB-FBSDE & $0.56$ & $0.80$ & $1.00$ & $1.00$ & $0.84$ & \citet{chen2022likelihood} \\
NLSB & $0.68$ & $0.84$ & $0.81$ & $0.79$ & $0.78$ & \citet{koshizuka2022neural} \\
OT-CFM & $0.78$ & $0.76$ & $0.77$ & $0.75$ & $0.77$ & \citet{tong2023improving} \\
WLF-OT & $0.65$ & $0.78$ & $0.76$ & $0.75$ & $0.74$ & \citet{neklyudov2024computational} \\
WLF-SB & $0.63$ & $0.79$ & $0.77$ & $0.74$ & $0.73$ & \citet{neklyudov2024computational} \\
JKOnet & $1.53$ & $1.27$ & $1.13$ & $1.41$ & $1.34$ & \citet{bunne2022proximal} \\
\midrule
\multicolumn{7}{l}{\textit{Static potential}} \\
JKOnet$^*_V$ & $0.99$ & $1.11$ & $1.06$ & $1.30$ & $1.12$ & \citet{terpin2024learning} \\
iJKOnet$_V$  & $0.92$ & $1.11$ & $0.95$ & $1.21$ & $1.05$ & \citet{persiianov2025ijkonet} \\
Stitching, $V^\theta(\x)$ & $\mathbf{0.46}$\stderr{0.01} & $\mathbf{0.60}$\stderr{0.01} & $\mathbf{0.60}$\stderr{0.01} & $\mathbf{0.65}$\stderr{0.01} & $\mathbf{0.58}$\stderr{0.01} & This paper \\
\midrule
\multicolumn{7}{l}{\textit{Time-varying potential}} \\
JKOnet$^*_{t,V}$ & $0.69$ & $0.77$ & $0.69$ & $0.78$ & $0.73$ & \citet{terpin2024learning} \\
iJKOnet$_{t,V}$  & $0.51$ & $0.58$ & $\mathbf{0.57}$ & $0.64$ & $0.58$ & \citet{persiianov2025ijkonet} \\
Stitching, $V^\theta(\x,t)$ & $\mathbf{0.44}$\stderr{0.01} & $\mathbf{0.56}$\stderr{0.01} & $\mathbf{0.57}$\stderr{0.01} & $\mathbf{0.60}$\stderr{0.02} & $\mathbf{0.54}$\stderr{0.01} & This paper \\
\bottomrule
\end{tabular}%
}
\end{table}

Under the harder leave-two-out protocol of \citet{shen2025mmsb} (train on $t\in\{0,2,4\}$, evaluate at held-out $t\in\{1,3\}$, $W_2$ metric), stitching also outperforms every published baseline: mean $W_2$ of $0.88$ for the time-varying variant against $0.92$ for the best published competitor (iJKOnet$_{t,V}$) and $1.12$ for iJKOnet$_V$ (full table in \autoref{app:rna_details}, \autoref{tab:rna_interp}). 

\subsection{Recovering interaction dynamics}
\label{sec:exp_interaction}

The mean-field limit ($N\to\infty$) of the WGF associated with \eqref{eq:functional} is the McKean–Vlasov process \citep{mckean1966, jabinMeanFieldLimit2017}. For finite $N$
 the corresponding interacting particle system is
\begin{equation}\label{eq:mckean-vlasov-sde}
    dX_t^i = -\nabla_{\x} V(X_t^i)dt - \frac{1}{N}\sum_{i\neq j} \nabla_{\x} W(X_t^i - X_t^j)dt + \sigma dB_t^i, \quad i=1,\dots,N.
\end{equation}
Such agent-based interacting particle models and their generalizations and extensions (for example, adding self-propulsion and alignment terms)  have been shown to approximate a wide variety of  collective behaviors across organisms and scales \citep{vicsek2012, ouellette2022physics, couzin2002, dorsogna2006, cucker2007}. 


In \eqref{eq:mckean-vlasov-sde}, $\nabla_{\x}V$ represents some environmental or intrinsic ``force'', while $\nabla_{\x}W$ encodes social interaction forces between the agents. The term $dB_t$ is the standard Wiener process.  Simultaneous recovery of both $V$ and $W$ from the marginals $X_t\sim q_t$ is a notoriously difficult problem \citep{wei2026, guan2024identifying, carrillo2025}, since these terms are mixed in the transient dynamics. Here we showcase the ability of our method to disentangle these terms, by simulating a 2D system with a confining ``Mexican hat'' potential $V(x){=}\alpha(\|x\|^2{-}\beta)^2$ and an attractive Gaussian kernel $W(r){=}\eta\,\smash{e^{-r^2}}$ (full setup in \autoref{app:cis}). The attractive interaction results in formation of one or more clusters \citep{rainer2002, motsch2014}, while the presence of $V$ confines the particles to a ring or radius $\sqrt{\beta}$.



We jointly learn an MLP $V^\theta(x)$ and radial $W^{\theta}(r)$, reporting scale-invariant pattern $R^2$ \citep{persiianov2025ijkonet} for the functionals and standard distributional metrics on held-out particles.

\begin{figure}[!htb]
\centering
\includegraphics[width=\linewidth]{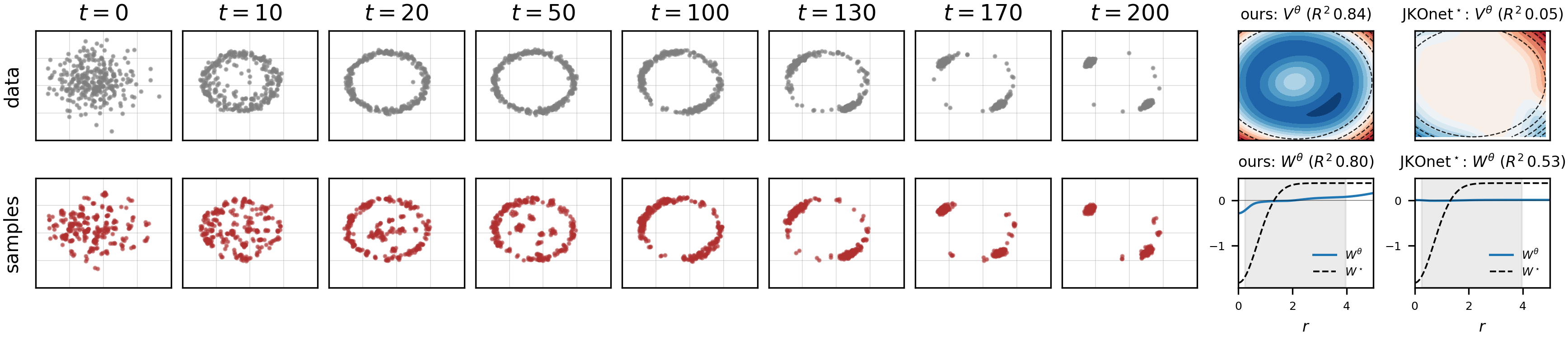}
\caption{\textbf{Stitching captures the Gaussian $\to$ ring $\to$ cluster phase transition and recovers $V, W$ at pattern $R^2{=}0.84, 0.80$.} \textbf{Top}: observed data at eight times; rightmost cell shows learned $V^\theta(x)$ filled, with $V$ contours dashed. \textbf{Bottom}: Stitching KDE samples; rightmost cell shows $W^\theta(r)$ (blue) vs.\ $W$ (dashed) over the data pair-distance band (shaded).}
\label{fig:cis_snapshots}
\end{figure}
\paragraph{Results.}
\autoref{fig:cis_snapshots} shows that stitching reproduces the phase transition from snapshots alone, and the rightmost column compares the recovered $V^\theta, W^\theta$ against the truth: pattern $R^2{=}0.84$ on $V$ and $0.81$ on $W$, on the same order as the strongest landscapes of \autoref{tab:synthetic2d}. \autoref{tab:cis} reports the distributional metrics; JKOnet$^\star$ has failed to recover $W$ completely, and is ${\sim}10\times$ slower per epoch since each JKO step requires an optimal transport coupling per snapshot pair ($200$ pairs in our example).
\begin{table}[!htb]
\centering
\caption{Interaction dynamics results on learning potential $V$, entropy and radial interaction $W$.}
\label{tab:cis}
\resizebox{\linewidth}{!}{%
\begin{tabular}{lcccccccr}
\toprule
Method & EMD\,$\downarrow$ & W$_2$\,$\downarrow$ & BW$^2$\,$\downarrow$ & MMD\,$\downarrow$ & $R^2(V)\,\uparrow$ & $R^2(W)\,\uparrow$ & per iter\,$\downarrow$ & total\,$\downarrow$ \\
\midrule
JKOnet$^\star_{V,W}$ \citep{terpin2024learning}, radial $W$     & $50.06$ & $82.62$ & $10120$ & $0.173$ & $0.04$ & $0.72$ & $\sim 4.1$\,s & $\sim 69$\,min \\
Stitching, $V^\theta(\x), W^\theta(r)$ (ours)                  & $\mathbf{0.49}$ & $\mathbf{0.85}$ & $\mathbf{0.19}$ & $\mathbf{0.017}$ & $\mathbf{0.84}$ & $\mathbf{0.81}$ & $\mathbf{0.24}$\,\textbf{s} & $\mathbf{4}$\,\textbf{min} \\
\bottomrule
\end{tabular}%
}
\end{table}

\subsection{Recovering non-gradient flows}
\label{subsec:non-grad}
As we discuss in further detail in \autoref{sec:conclusion} below, our method can in principle be applied to learning more general flows beyond WGFs. Here we showcase this ability by fitting a \emph{chiral} dynamics \citep{liebchen2022} with the stitching loss. The dynamics is interaction-only, such that the kernel has a nonzero curl. 
In \autoref{fig:chiral_snapshots} we show the original data and the reconstructed stitching trajectories over a few time snapshots. Full details are provided in 
\autoref{app:nongradient}.

\begin{figure}[htb]
\centering
\includegraphics[width=\linewidth]{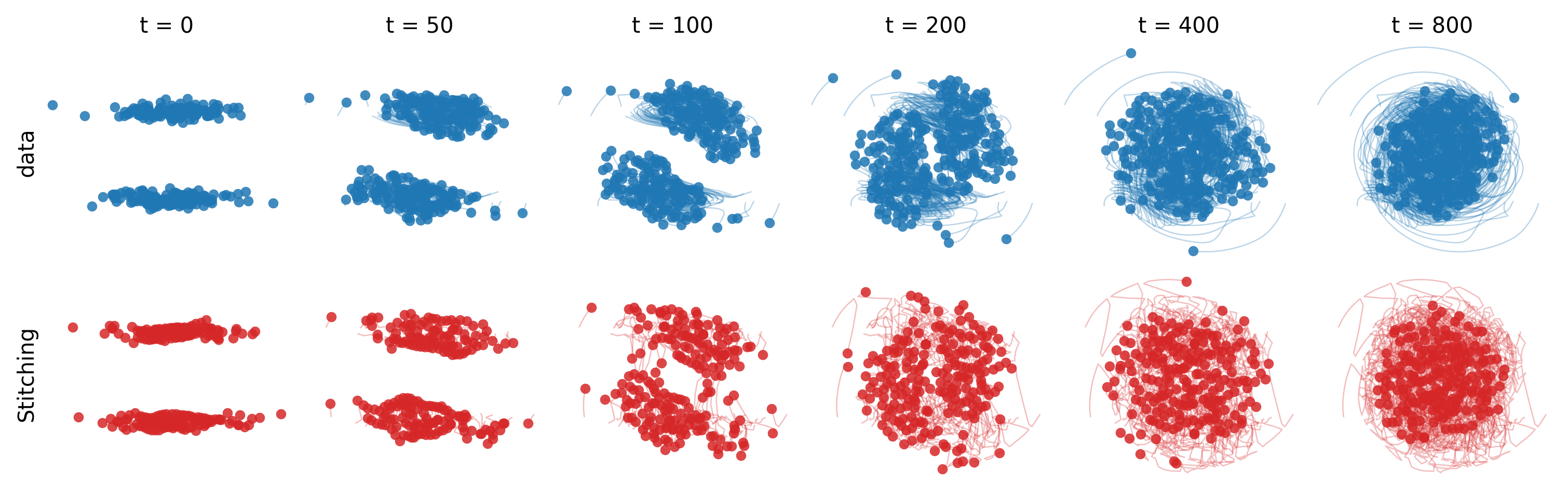}
\caption{\textbf{Stitching tracks the chiral orbiting dynamics.} Samples from the learned curve $\rho^\theta$ (model) and the data $q$ (observed) at particular time snapshots over the training window.}
\label{fig:chiral_snapshots}
\end{figure}





%

\section{Related work}
\label{sec:related_work}

In this section, we discuss some related approaches in learning potential functions for population dynamics in the context of Wasserstein gradient flows.  \autoref{app:related_work} provides a more detailed review.

\paragraph{JKO-based methods.} 
The dominant algorithmic framework to tackle our main goal is the Jordan--Kinderlehrer--Otto (JKO) scheme~\citep{jordan1998variational}, which solves a sequence of proximal problems with a Wasserstein penalty, cf. \autoref{app:wgf_background}.
Currently, there are two strategies to learn the potential $\calF$ from data $\{q_t\}_{t\in \calT_{\mathrm{obs}}}$ using the JKO scheme. 

JKOnet~\citep{bunne2022proximal} learns $\calF$ by backpropagation through the inner loop that solves the proximal problem.
The optimal transport is represented in terms of a transport map $\nabla \psi^\theta$ where $\psi^\theta: \R^D \to \R$ is an input-convex neural network (ICNN;~\citet{amos2017input}). 
iJKOnet~\citep{persiianov2025ijkonet} avoids the need for ICNNs through a min-max formulation.
JKOnet$^\star$~\citep{terpin2024learning} algorithms solve optimal transport of the empirical data once upfront,
and then fit the gradient of the functional to the observed displacements.
The above strategies directly represent $\calF$, but leave the curve $\rho$ defined implicitly through the JKO scheme on $\calF$.
In contrast, we parametrize both $\calF$ and $\rho$, coupled by a residual loss to enforce consistency.
Directly representing $\rho$ allows us to choose a temporal discretization independent from that of the observations.

\paragraph{Residual Losses.}
Concurrent work of \citet{liu2026generative} considers the problem of solving a WGF given a known functional $\calF$ using a form of \eqref{eq:velocity_res}, resulting in \emph{GenWGP}. In particular, considering a $K$-step discretization of the interval $[0, T]$, GenWGP parameterizes the flow-map $\Phi^{\theta}(t, \x)$ as a normalizing flow. This flow is applied to a population of particles, and finite difference methods are then used to approximate the velocity residual, which becomes the objective in gradient descent. 
GenWGP requires solving a costly neural ODE via application of the learned flow-map, but sidesteps side-steps the KDE approximation that stitching must take.

\paragraph{Action Matching.}

\citet{neklyudov2023action} introduces action matching (AM) for learning the velocity field of a population dynamic, and \citet{neklyudov2024computational} extends the framework to a broader class of Wasserstein Lagrangian flows beyond pure gradient flows. AM fits a velocity field to data without access to ground-truth velocities. Adapted to our setting, the goal is to learn
a functional whose gradient
matches the velocity of the data. \citet{neklyudov2024computational} effectively integrate \eqref{eq:velocity_res} by parts to eliminate the dependence on the velocity.
Unlike the divergence-based methods, AM fuses residual and data fitting into a single objective. The price is that AM requires data at the temporal endpoints and learns only the functional, leaving the trajectory $\rho$ implicit.


\section{Discussion}
\label{sec:conclusion}
\paragraph{Summary} We introduce the Wasserstein residual framework as way to learn Wasserstein gradient flows from snapshots of population dynamics. We instantiate the Wasserstein residuals framework by the stitching algorithm, which is a simulation-free particle method robust against long temporal gaps in the observed time-series. The stitching achieves state-of-the-art results on the embryoid body (EB) single-cell RNA sequencing dataset.



\paragraph{Limitations}
Being a particle method, stitching inherits the $O(N^2)$ quadratic computation per time step for the entropy and interaction terms, while the potential term is $O(N)$ linear. In addition, theoretical convergence guarantees as $N\to\infty$ remain open. Finally, the residual framework enforces a constraint using a residual regularizer, which does not guarantee the WGF condition $\calR = 0$. In practice the learned flows have residuals close to zero.


\paragraph{Broader impacts}
Stitching enables learning population dynamics from sparse snapshots in domains ranging from single-cell biology and trajectory inference to crowd and collective behaviour. The most immediate benefits are in scientific discovery with dual-use risks generic to dynamical-system inference: applied uncritically to social or behavioural data, the recovered potentials may be taken as causal explanations. 
Practitioners should validate the gradient-flow assumption against domain knowledge before drawing conclusions from the learned functional.

\subsection{Future work}
\paragraph{Density residuals.} In this work we focus on \eqref{eq:velocity_res} and instantiate it with our stitching method. Future research should explore the use of \eqref{eq:density_res}, which can be exploited using neural networks or particle-based methods.

\paragraph{Beyond Wasserstein gradient flows.} The residual approach applies in fact to a more general family of flows, beyond Wasserstein gradient flows. We demonstrate this point in \autoref{subsec:non-grad}, but the framework applies much more generally. Suppose our goal is to match the data  $\{q_t\}_{t\in\calT_{\mathrm{obs}}}$ to a curve $\rho=(\rho_t)$ of the form $
\partial_t\rho_t(\x) = - \Div\big(\rho_t(\x)\u(\x)\big)$,
where $\u$ is a vector field for which we assume some structure. For example, in this work we assume $\smash{\u_t(\x)=-\nabla_{\x}\dFdrt(\x)} $ for some functional $\calF$. To find $\u$ we can use either
\eqref{eq:density_res} or
\eqref{eq:velocity_res} by replacing $\smash{\nabla_{\x}\dFdrt(\x)}$ by a parametrized family $\u_t$. Such non-gradient-flows arise in the study of  transformers \citep{geshkovski2025mathematical},  chiral active matter \citep{liebchen2022}, non-reciprocal collection systems \citep{fruchart2021}, and ocean currents \citep{petrovic2025curlya}.

\section*{Disclosure}
Yair Shenfeld's and Ricardo Baptista's contributions to this work resulted in part from their affiliation with Basis Research Institute (an outside organization with
respect to Brown University and University of Toronto).

\bibliography{ottograd}

\appendix
\newpage

\section{Mathematical background: Wasserstein gradient flows}
\label{app:wgf_background}

This appendix collects standard definitions and results for Wasserstein gradient flows used in the main text. We refer the reader to \citet{ambrosio2005gradient,villani2021topics,santambrogio2015optimal} for comprehensive treatment.

\subsection{The Wasserstein space}
\label{subapp:wasserstein_space}

The space of probability measures with finite second moment is
\begin{equation}
\calP_2(\R^D) := \left\{\rho \text{ probability measure on } \R^D : \int_{\R^D} \left\|\x\right\|^2\,\mathrm{d}\rho(\x) < \infty\right\},
\end{equation}
endowed with the 2-Wasserstein distance
\begin{equation}
\label{eq:W_intro}
W_2^2(\rho,\nu) := \inf_{\gamma\in\Pi(\rho,\nu)}\int_{\R^D\times\R^D}\left\|\x-\x'\right\|^2\,\mathrm{d}\gamma(\x,\x'),\qquad \rho,\nu\in\calP_2(\R^D),
\end{equation}
where $\Pi(\rho,\nu)$ is the set of couplings (probability measures on $\R^D\times\R^D$ with marginals $\rho$ and $\nu$). An absolutely continuous curve $\rho:[0,T]\to\calP_2(\R^D)$ admits a unique minimal velocity field $\v$ satisfying the continuity equation
\begin{equation}
\label{eq:euler_app}
\partial_t\rho_t(\x) + \Div\bigl(\rho_t(\x)\v_t(\x)\bigr) = 0,
\end{equation}
\cite[p.~167]{santambrogio2015optimal}; \eqref{eq:euler_app} is the \emph{Eulerian} description. The equivalent \emph{Lagrangian} description is via particle trajectories
\begin{equation}
\label{eq:lagrange_app}
\dot\x_t = \v_t(\x_t),\qquad \x_0\sim\rho_0,
\end{equation}
which satisfies \cite[Theorem 8.3.1]{santambrogio2015optimal}
\begin{equation}
\label{eq:euler-lagrange_app}
\x_t\sim\rho_t,\qquad t\in[0,T].
\end{equation}

\subsection{Functionals on the Wasserstein space}
\label{subapp:examples}
A functional on the Wasserstein space is a map $\calF: \calP_2(\R^D)\to \R$. A general family of functionals, common in applications, is 
\begin{equation}
\label{eq:functional_app}
\calF[p] = \int_{\R^D} \pot(\x)\,\mathrm{d}p(\x) + \int_{\R^D}\pote(p(\x))\,\mathrm{d}\x + \int_{\R^D\times\R^D}\inter(\x'-\x)\,\mathrm{d}p(\x)\,\mathrm{d}p(\x'),
\end{equation}
where $p\in \calP_2(\R^D)$, $\pot:\R^D\to\R$, $\pote:\R_{\ge 0}\to \R$, and $\inter:\R^D\to \R$ are sufficiently regular. Let us consider some concrete examples. 

\begin{example}[Kullback--Leibler divergence]
\label{ex:KL}
Let $\pot$ be such that $\nu = e^{-\pot}$ is a probability measure in $\calP_2(\R^D)$, let $\pote(r)=r\log r$, and let  $\inter = 0$. Then \eqref{eq:functional_app} is the Kullback--Leibler divergence functional
\begin{equation}
\label{eq:func_kl}
\calF[p] = \KL[p\|\nu] = \int_{\R^D} \log\left(\frac{\rho(\x)}{\nu(\x)}\right)\,\mathrm{d}\rho(\x).
\end{equation}
\end{example}

\begin{example}[Aggregation]
\label{ex:aggregation}
Let $\pot=0$, $\pote=0$, and let $\inter$ be a symmetric interaction kernel. Then 
\begin{equation}
\label{eq:func_agg}
\calF[p] = \int_{\R^D\times\R^D}\inter(\x'-\x)\,\mathrm{d}p(\x)\,\mathrm{d}p(\x')
\end{equation}
Functionals of the form \eqref{eq:func_agg} model swarming, chemotaxis, and granular media \cite[\S 5.4]{villani2021topics}.
\end{example}

\subsection{Gradient flows}
\label{subapp:agc_four}

The classical reference \citet[Ch.~11, p.~279--280]{ambrosio2005gradient} identifies four equivalent formulations of gradient flows in metric spaces: the \emph{Tangent} condition, the \emph{Jordan--Kinderlehrer--Otto} (JKO) scheme, the \emph{Evolution Variational Inequality} (EVI), and the \emph{Energy Dissipation Equality} (EDE). To build intuition, we begin in Euclidean space, then map each formulation to its Wasserstein-space analogue.

 Let $f:\R^D\to\R$ be a smooth function and let $(x_t)_{t\in[0,T]}$ a smooth curve in $\R^D$.

\paragraph{Tangent condition.} The most direct definition of $(x_t)_{t\in [0,T]}$ being a gradient flow of $f$ is if the equation
\begin{equation}
\label{eq:tangent_euclid}
\tag{Tangent$_E$}
\dot x_t = -\nabla f(x_t) \qquad \text{for all } t\in[0,T],
\end{equation}
holds. The terminology comes from $\dot x_t$ lying in the tangent space at $x_t$. While intuitive, \eqref{eq:tangent_euclid} cannot accommodate non-differentiable $f$.

\paragraph{JKO scheme.} A more general definition of $(x_t)_{t\in [0,T]}$ being a gradient flow of $f$ if it satisfies 
\begin{equation}
\label{eq:JKO_euclid}
\tag{JKO$_E$}
x_{t+h} = \argmin_{x\in\R^D}\left[ f(x) + \frac{1}{2h}\| x - x_t\|^2 \right]
\qquad \text{for all } h>0.
\end{equation}
For differentiable $f$, the first-order optimality condition of \eqref{eq:JKO_euclid} reads
\begin{equation}
\label{eq:JKO_FO_euclid}
\tag{JKO-FO$_E$}
\frac{x_{t+h}-x_t}{h} = -\nabla f(x_{t+h}) \qquad \text{for all } h>0,
\end{equation}
which is the implicit-Euler discretization of \eqref{eq:tangent_euclid} and recovers it as $h\to 0$. Equation \eqref{eq:JKO_euclid} is the basis of \emph{proximal} algorithms, and the terminology comes from the seminal work of Jordan--Kinderlehrer--Otto (JKO)  \citep{jordan1998variational} who used this formulation to define gradient flows in Wasserstein space. 

\paragraph{Evolution Variational Inequality (EVI).} The EVI definition for $(x_t)_{t\in [0,T]}$ being a gradient flow of $f$ requires the existence of $\alpha\in\R$ such that
\begin{equation}
\label{eq:EVI_euclid}
\tag{EVI$_E$}
\frac{1}{2}\frac{\mathrm{d}}{\mathrm{d}t}\| x_t - y\|^2 \le f(y) - f(x_t) - \frac{\alpha}{2}\| x_t - y\|^2
\qquad \text{for all } y\in\R^D.
\end{equation}
For smooth $\alpha$-convex $f$, the convexity inequality $\nabla f(x_t)\cdot(y - x_t) \le f(y) - f(x_t) - \tfrac{\alpha}{2}\| x_t - y\|^2$ combined with \eqref{eq:EVI_euclid} forces $\dot x_t = -\nabla f(x_t)$. The EVI form is central to the gradient flow theory in metric spaces \citep{ambrosio2005gradient}.

\paragraph{Energy Dissipation Equality (EDE).} The EDE definition for $(x_t)_{t\in [0,T]}$ being a gradient flow of $f$ requires the validity of the identity
\begin{equation}
\label{eq:EDE_euclid}
\tag{EDE$_E$}
f(x_T) - f(x_0) = -\int_0^T \left[\frac{1}{2}\|\nabla f(x_r)\|^2 + \frac{1}{2}\|\dot x_r\|^2\right]\mathrm{d}r.
\end{equation}
To see why \eqref{eq:EDE_euclid} characterizes gradient flows, note that for any smooth curve, the chain rule and successive applications of Cauchy--Schwarz and AM--GM inequalities give
\begin{align*}
f(x_T) - f(x_0) &= \int_0^T \nabla f(x_r)\cdot\dot x_r\,\mathrm{d}r
\;\ge\; -\int_0^T \|\nabla f(x_r)\|\cdot\|\dot x_r\|\,\mathrm{d}r\\
&\;\ge\; -\int_0^T \left[\frac{1}{2}\|\nabla f(x_r)\|^2 + \frac{1}{2}\|\dot x_r\|^2\right]\mathrm{d}r.
\end{align*}
The first inequality is an equality if and only if $\dot x_r = -c_r\,\nabla f(x_r)$ for some $c_r\ge 0$, while the second inequality is an equality if and only if  $\|\dot x_r\| = \|\nabla f(x_r)\|$. Together (when $\nabla f(x_r) \ne 0$) the two equalities force $c_r = 1$, i.e.\ $\dot x_t = -\nabla f(x_t)$. It follows that \eqref{eq:EDE_euclid} holding with equality is equivalent to the tangent condition.

\paragraph{From Euclidean to Wasserstein.} The four formulations transfer to the Wasserstein space $\calP_2(\R^D)$ by replacing the Euclidean structure with the $W_2$ metric. Given a curve $(\rho_t)_{t\in [0,T]}$ in the Wasserstein space a functional $\calF:\calP_2(\R^D)\to \R$ the following four formulations define what it means for $(\rho_t)_{t\in [0,T]}$  to be a gradient flow of $\calF$. 
\begin{itemize}
\itemsep0.2em
\item \eqref{eq:tangent_euclid} becomes the continuity equation 
\begin{equation}
\label{eq:cont_eq_4grad}
\tag{Tangent$_W$}
\partial_t\rho_t(\x) = -\Div(\rho_t(\x)\v_t(\x))\quad \textnormal{where} \quad \v_t(\x) = -\nabla_\x\dFdrt(\x), 
\end{equation} 
which is the  constraint underlying the density residual \eqref{eq:density_res}.
\item \eqref{eq:JKO_euclid} becomes the Wasserstein proximal scheme 
\begin{equation}
\label{eq:JKO_4grad}
\tag{JKO$_W$}
\rho_{t+h} = \argmin_{\nu\in \calP_2(\R^D)}\left[\calF(\nu)+\frac{1}{2h}W_2^2(\rho_t,\nu)\right],
\end{equation}
which is the basis of the JKO methods; cf. \autoref{subsec:jko}. 
\item \eqref{eq:EVI_euclid} becomes a $W_2$-EVI
\begin{equation}
\label{eq:EVI_W}
\tag{EVI$_W$}
\frac{1}{2}\frac{\mathrm{d}}{\mathrm{d}t}W_2^2(\rho_t, \nu) \le \calF(\nu) - \calF(\rho_t) - \frac{\alpha}{2}W_2^2(\rho_t ,\nu)
\quad \text{for all } \quad \nu\in\calP_2(\R^D),
\end{equation}
which is
 an inequality rather than an equality, so is not amenable to residual minimization.
 
\item \eqref{eq:EDE_euclid} becomes the Wasserstein EDE, 
\begin{equation}
\label{eq:EDE_4grad}
\tag{EDE$_W$}
\calF[\rho_T]-\calF[\rho_0]= - \int_0^T \!\left[\frac{1}{2}|\dot\rho_t|^2_{W_2}
+\frac{1}{2}\int_{\R^D}\!\left\|\nabla_{\x}\dFdrt(\x)\right\|^2\!\rho_t(\x)\,\mathrm{d}\x\right]\mathrm{d}t,
\end{equation}
where 
\[
|\dot\rho_t|^2_{W_2} := \lim_{h\to 0}\frac{W_2^2(\rho_t,\rho_{t+h})}{h^2}. 
\]
Lemma \ref{lem:vel_dissipation} below shows how to reformulate \eqref{eq:EDE_4grad} in a way which is amenable to a residual formulation, which is the basis for the velocity residuals of Section \ref{sec:vel_res} and the works  \citet{liu2026generative} and \citet{neklyudov2023action}. 
\end{itemize}


\begin{lemma}[EDE form of $\calR_{\mathrm{vel}}$]
\label{lem:vel_dissipation}
Let $(\rho,\v)$ be a curve satisfying the continuity equation 
\begin{equation}
\label{eq:con_tq_lem}
\partial_t\rho_t(\x) = - \Div\big(\rho_t(\x)\v_t(\x)\big),
\end{equation}

such that, for all $t\in [0,T]$, $\rho_t\in C^1(\R^D)$ vanishes at infinity, and $\v_t\in C^1(\R^D,\R^D)$. Let $\calF:\calP_2(\R^D)\to \R$ be a functional such that, for all $t\in [0,T]$, $\calF(\rho_t)<\infty$, $\dFdrt\in C^1(\R^D)$ vanishes at infinity, and $\calF$ satisfies the chain rule
\begin{equation}
\label{eq:chain_rule}
\calF(\rho_T)-\calF(\rho_0)=\int_0^T\!\!\int_{\R^D} \dFdrt(\x)\,\partial_t\rho_t(\x)\,\mathrm{d}\x\,\mathrm{d}t.
\end{equation}
Then,
\begin{equation}
\label{eq:EDE_obj_intro_EQUIV_proof}
\begin{split}
\calR_{\mathrm{vel}}[\calF, (\rho, \v)]&=\int_0^T\!\!\int_{\R^D}
  \Bigl\|\v_t(\x) + \nabla_{\x}\dFdrt(\x)\Bigr\|^2
  \rho_t(\x)\,\mathrm{d}\x\,\mathrm{d}t\\
  &=2\left(\calF[\rho_T]-\calF[\rho_0]\right) + \int_0^T \!\left[|\dot\rho_t|^2_{W_2}
+\int_{\R^D}\!\left\|\nabla_{\x}\dFdrt(\x)\right\|^2\!\rho_t(\x)\,\mathrm{d}\x\right]\mathrm{d}t.
\end{split}
\end{equation}
\end{lemma}

\begin{proof}
By \citet[Theorem 8.3.1]{ambrosio2005gradient}, $\left|\dot\rho_t\right|^2_{W_2}=\int_{\R^D}\left\|\v_t(\x)\right\|^2\rho_t(\x)\,\mathrm{d}\x$. Expanding the squared norm in $\calR_{\mathrm{vel}}[\calF, (\rho, \v)]$, and integrating by parts,
\begin{align*}
&\calR_{\mathrm{vel}}[\calF, (\rho, \v)]\\
&=\int_0^T\!\!\int_{\R^D}\left\|\v_t\right\|^2\rho_t\,\mathrm{d}\x\,\mathrm{d}t
+ 2\int_0^T\!\!\int_{\R^D}\v_t\cdot\nabla_\x\dFdrt\,\rho_t\,\mathrm{d}\x\,\mathrm{d}t
+\int_0^T\!\!\int_{\R^D}\left\|\nabla_\x\dFdrt\right\|^2\rho_t\,\mathrm{d}\x\,\mathrm{d}t \\
&=\int_0^T \left|\dot\rho_t\right|^2_{W_2}\mathrm{d}t
- 2\int_0^T\!\!\int_{\R^D}\dFdrt\Div(\rho_t\v_t)\,\mathrm{d}\x\,\mathrm{d}t
+\int_0^T\!\!\int_{\R^D}\left\|\nabla_\x\dFdrt\right\|^2\rho_t\,\mathrm{d}\x\,\mathrm{d}t.
\end{align*}
By \eqref{eq:con_tq_lem} and the chain rule \eqref{eq:chain_rule},
\[
-\int_0^T\!\!\int_{\R^D}\dFdrt\Div(\rho_t\v_t)\,\mathrm{d}\x\,\mathrm{d}t
= \int_0^T\!\!\int_{\R^D}\dFdrt\,\partial_t\rho_t\,\mathrm{d}\x\,\mathrm{d}t
= \calF(\rho_T)-\calF(\rho_0),
\]
which establishes \eqref{eq:EDE_obj_intro_EQUIV_proof}.
\end{proof}

\section{Divergences}
\label{app:divergences}

In this section we expand on the divergence options $\calD$ used in the data-fitting term.

\paragraph{Kullback--Leibler (KL)}
Likelihood maximization leads to
\begin{equation}
\calD_{\mathrm{KL}}(\rho_t,q_t):=-\int_{\R^D}\!(\log\rho_t(\x))q_t(\x)\mathrm{d}\x
=\KL(q_t\|\rho_t)+\int_{\R^D}\!(\log q_t(\x))q_t(\x)\mathrm{d}\x,
\end{equation}
which can be approximated as 
\begin{equation}
\calD_{\mathrm{KL}}(\rho_t,q_t)\approx -\sum_{i=1}^{N_t} \log\rho_t(\y_{t,i}),\qquad \textnormal{where}\qquad \{\y_{t,i}\}_{i=1}^{N_t}\overset{\textnormal{i.i.d.}}{\sim} q_t. 
\end{equation}
The divergence $\calD_{\mathrm{KL}}(\rho_t,q_t)$  can be used whenever $\rho_t$ can be evaluated pointwise.

\paragraph{Score matching \citep{hyvarinen2005estimation}.}
The score matching cost
\begin{equation}
\label{eq:sm}
\int_{\R^D}\|\nabla\log \rho_t(\x)-\nabla\log q_t(\x)\|^2q_t(\x)\mathrm{d}\x,
\end{equation}
whose minimization over $\rho$ is equivalent to minimizing over $\rho$,
\begin{equation}
\label{eq:SM}
\calD_{\mathrm{SM}}(\rho_t,q_t):=\int_{\R^D}\!\Big[\left\|\nabla\log\rho_t(\x)\right\|^2+2\Delta\log\rho_t(\x)\Big]q_t(\x)\mathrm{d}\x,
\end{equation}
leads to the divergence
\begin{equation}
\calD_{\mathrm{SM}}(\rho_t,q_t)\approx \sum_{i=1}^{N_t} [\left\|\nabla\log\rho_t(\y_{t,i})\right\|^2+2\Delta\log\rho_t(\y_{t,i})],\qquad \textnormal{where}\quad \{\y_{t,i}\}_{i=1}^{N_t}\overset{\textnormal{i.i.d.}}{\sim} q_t.
\end{equation}
The divergence $\calD_{\mathrm{SM}}(\rho_t,q_t)$ can be used whenever $\nabla\log\rho_t$ is available, potentially with the extra cost computing $\nabla\log \rho_t$ is only $\log \rho_t$ is parametrized. 

\paragraph{Denoising score matching \citep{vincent2011connection}.}
Let $q_t^\sigma:=\int \kappa_\sigma(\cdot|\z)q_t(\z)\mathrm{d}\z$ with $\kappa_\sigma(\cdot|\z)=\N(\z,\sigma^2 \mathrm{I}_D)$. Minimizing over $\rho_t$ the score matching cost between $\rho_t$ and $q_t^\sigma$
\begin{equation}
\label{eq:dsm}
\int_{\R^D}\|\nabla\log \rho_t(\x)-\nabla\log q_t^\sigma(\x)\|^2q_t(\x)\mathrm{d}\x,
\end{equation}
is equivalent to minimizing over $\rho_t$, 
\begin{equation}
\calD_{\mathrm{DSM}}(\rho_t,q_t):=\int_{\R^D}\!\!\int_{\R^D}\!\|\nabla\log\rho_t(\z)-\nabla_{\z}\log \kappa_\sigma(\z|\x)\|^2 \kappa_\sigma(\z|\x)q_t(\x)\mathrm{d}\z\mathrm{d}\x. 
\end{equation}
The denosing score matching divergence $\calD_{\mathrm{DSM}}(\rho_t,q_t)$ can be 
approximated by
\begin{equation}
\calD_{\mathrm{DSM}}(\rho_t,q_t)\approx \sum_{i=1}^{N_t} \|\nabla\log\rho_t(\z_{t,i})-\nabla_{\z}\log \kappa_\sigma(\z_{t,i}|\y_{t,i})\|^2,\, \,\z_{t,i}\sim \kappa_\sigma(\cdot|\y_{t,i}),\, \{\y_{t,i}\}_{i=1}^{N_t}\overset{\textnormal{i.i.d}}{\sim}q_t. 
\end{equation}
In contrast to \eqref{eq:SM}, denoising score matching $\calD_{\mathrm{DSM}}$ avoids the computation of the Laplacian in $\calD_{\mathrm{SM}}$ at the cost of fitting the smoothed $q_t^\sigma$ rather than $q_t$.



\section{Stitching: details}
\label{app:stitching_derivation}


In this section we derive the velocity expression  \eqref{eq:kde_vel} and explain the reasoning behind \eqref{eq:stitch_ideal_global_approx}. We start with Equation \eqref{eq:kde_vel}.

\begin{claim}
\label{cl:vec_KDE}
 Let $\rho=(\rho_t)_{t\in [0,T]}$ be of the form $\rho_t(\x) = \sum_{k=1}^N w_k \,\kernel(\x - \x_{t,k})$, where for each $k=1,\ldots, N$, the trajectory $[0,T]\ni t\mapsto \x_{t,k}$ is differentiable. Then, $(\rho,\v)$ satisfies the continuity equation \eqref{eq:cont_eq_intro} with
\begin{equation}
\label{eq:vec_KDE}
\v_t(\x)=\frac{\sum_{k=1}^N w_k \,\kernel(\x - \x_{t,k}) \,\dot\x_{t,k}}{\sum_{l=1}^N w_l\,\kernel(\x - \x_{t,l}) }.
\end{equation}
\end{claim}
\begin{proof}
Since $\nabla_{\x_{t,k}}\kernel(\x - \x_{t,k})=-\nabla_{\x}\kernel(\x - \x_{t,k})$ we have
\begin{align*}
\partial_t\rho_t(\x)& = \sum_{k=1}^N w_k \,\partial_t\kernel(\x - \x_{t,k})=  \sum_{k=1}^N w_k \,\nabla_{\x_{t,k}}\kernel(\x - \x_{t,k}) \cdot \dot\x_{t,k}\\
&=-\nabla_{\x}\left(\sum_{k=1}^N w_k \,\kernel(\x - \x_{t,k}) \dot\x_{t,k}\right)=-\nabla_{\x}\left(\rho_t(\x)\frac{\sum_{k=1}^N w_k \,\kernel(\x - \x_{t,k}) \dot\x_{t,k}}{\sum_{l=1}^Nw_l\,\kernel(\x - \x_{t,l}) }\right).
\end{align*}
\end{proof}
Next we explain the reasoning behind \eqref{eq:stitch_ideal_global_approx}. 

\begin{claim}
\label{cl:cont_eq_particle}
Let $\rho_t=\sum_{k=1}^N w_k\delta_{\x_{t,k}}$ with differentiable trajectories $t\mapsto \x_{t,k}$, and let $\v$ be the unique minimal velocity satisfying \eqref{eq:cont_eq_intro} in the distributional sense. Then $\dot\x_{t,k}=\v_t(\x_{t,k})$ for each $k$.
\end{claim}
\begin{proof}
For any test function $\eta:[0,T]\times\R^D\to\R$, weak satisfaction of the continuity equation gives
\[
\int_s^{s+h}\!\!\int [\partial_t\eta+\nabla\eta\cdot \v_t]\mathrm{d}\rho_t\mathrm{d}t = \int\eta(s+h,\cdot)\mathrm{d}\rho_{s+h}-\int\eta(s,\cdot)\mathrm{d}\rho_s.
\]
Substituting $\rho_t=\sum_{k=1}^Nw_k\delta_{\x_{t,k}}$, and choosing $\eta$ localized around a single $\x_{t,k}$, yields $\dot\x_{t,k}=\v_t(\x_{t,k})$.
\end{proof}

\paragraph{From the velocity residual to the boxed objective.}
Substituting the KDE parametrization \eqref{eq:particle_param} into the velocity residual \eqref{eq:velocity_res} and approximating the spatial expectation at the centers (exact as $\kernel \to \delta$) gives
\[
\calR_{\mathrm{vel}}[\calF^\theta, \rho^\theta, \v^\theta] \;\approx\; \int_0^T \sum_{k=1}^N \wt_k\,\left\|\v_t^\theta(\x_{t,k}^\theta) + \nabla_\x\dFdrtt(\x_{t,k}^\theta)\right\|^2 \,\mathrm{d}t.
\]
By Claim \ref{cl:cont_eq_particle}, $\v_t^\theta(\x_{t,k}^\theta) = \dot\x_{t,k}^\theta$ in this limit. Discretising time on $0 = t_0 < \cdots < t_{K-1} = T$ with forward Euler $\dot\x_{t_j,k}^\theta \approx \Delta\x_{t_{j+1},k}^\theta / \Delta t_{j+1}$ (where $\Delta\x_{t_{j+1},k}^\theta := \x_{t_{j+1},k}^\theta - \x_{t_j,k}^\theta$), replacing the time integral by the left-Riemann sum $\int_0^T f(t)\,\mathrm{d}t \approx \sum_j \Delta t_{j+1}\,f(t_j)$, and applying the algebraic identity $h\,\|a/h + b\|^2 = (1/h)\,\|a + h\,b\|^2$ on each summand yields the boxed displacement form \eqref{eq:stitch_ideal_global_approx}.

\paragraph{Computation of the score term.}
The KDE \eqref{eq:particle_param} has analytic score
\begin{equation}
\label{eq:kde_score}
\nabla_\x\log\rho_t^\theta(\x) \;=\; \frac{\sum_{k=1}^N \wt_k\,\nabla_\x\kernel(\x - \x_{t,k}^\theta)}{\sum_{l=1}^N \wt_l\,\kernel(\x - \x_{t,l}^\theta)},
\end{equation}
which we need to evaluate at the centers $\x_{t,j}^\theta$. The numerator self-term $\wt_j \nabla\kernel(\mathbf{0}) = \mathbf{0}$ is harmless. The denominator self-term $\wt_j\kernel(\mathbf{0})$ is not: it is the kernel's peak value, and it stays large regardless of where the other particles sit. When the neighbours are far compared to the bandwidth, the numerator (a sum of tiny neighbour kernels and their gradients) is already small, and dividing by an over-strong self-normaliser collapses the score toward $\mathbf{0}$. The entropy contribution then drops out of the velocity residual and particles collapse onto minima of $V^\theta$ instead of spreading.

We drop particle $j$ from both numerator and denominator when evaluating the score at $\x_{t,j}^\theta$:
\begin{equation}
\label{eq:kde_score_loo}
\nabla_\x\log\rho_t^\theta(\x_{t,j}^\theta) \;\approx\; \frac{\sum_{k\neq j} \wt_k\,\nabla_\x\kernel(\x_{t,j}^\theta - \x_{t,k}^\theta)}{\sum_{l\neq j} \wt_l\,\kernel(\x_{t,j}^\theta - \x_{t,l}^\theta)}.
\end{equation}
The two small quantities now divide to give a finite vector pointing toward the nearest neighbours, restoring the diffusion pressure of the entropy term \citep{wand1994kernel}.

\paragraph{Centers vs.\ KDE-Monte-Carlo quadrature.}
The boxed objective evaluates the residual integrand at the particle centers $\{\x_{t,k}^\theta\}$. The implementation also supports a stochastic alternative: for each particle, draw a perturbation $\varepsilon_{t,k} \sim \kernel(\mathbf{0}, h^2 \mathrm{I})$ from the KDE kernel and evaluate the residual at $\y_{t,k} = \x_{t,k}^\theta + \varepsilon_{t,k}$, with the velocity at $\y_{t,k}$ given by the Nadaraya--Watson estimator \eqref{eq:vec_KDE}. The two estimators are combined convexly through a parameter $\alpha \in [0,1]$:
\begin{equation}
\label{eq:alpha_quadrature}
\widehat\calR_{\mathrm{vel}}^\alpha \;:=\; (1-\alpha)\,\widehat\calR_{\mathrm{vel}}^{\mathrm{centers}} \;+\; \alpha\,\widehat\calR_{\mathrm{vel}}^{\mathrm{KDE\text{-}MC}}.
\end{equation}
At $\alpha = 0$ the residual is the deterministic centers approximation: exact in the small-bandwidth limit, but blind to off-trajectory information. At $\alpha = 1$ it is an unbiased Monte-Carlo estimator under the KDE measure, at the cost of MC variance and a Nadaraya--Watson regression bias of order $O(h^2)$. Intermediate $\alpha$ trades the two biases against each other. We use $\alpha = 0$ by default in all experiments. We found the $\alpha = 0.5$ to give slightly better results in the low-data regime wavy-valley illustration.

\paragraph{Time-discretisation schemes.}
The forward-Euler discretisation $\dot\x_{t_j,k}^\theta \approx \Delta\x_{t_{j+1},k}^\theta / \Delta t_{j+1}$ used above is one of several supported choices. The implementation also exposes:
\begin{itemize}\itemsep0.2em
    \item \emph{Backward (implicit) Euler}: Evaluates $\nabla_\x\dFdrtt$ at the right endpoint $\x_{t_{j+1},k}^\theta$ rather than $\x_{t_j,k}^\theta$.
    \item \emph{Midpoint}: Evaluates both the displacement and $\nabla_\x\dFdrtt$ at $t_{j+1/2}$, $\dot\x_{t_{j+1/2},k}^\theta \approx \Delta\x_{t_{j+1},k}^\theta / \Delta t_{j+1}$. $O(h^2)$ accurate per step.
    \item \emph{Trapezoidal}: Averages the forward and backward Euler residuals, evaluating $\nabla_\x\dFdrtt$ at both endpoints. $O(h^2)$ accurate; this is the default in our experiments and the natural choice on non-uniform grids.
\end{itemize}
On uniform grids the schemes differ only in higher-order error terms; the trade-off is between accuracy (midpoint, trapezoidal) and per-step cost (forward Euler avoids re-evaluating $\nabla_\x\dFdrtt$ at the right endpoint).

\section{Extended Related Work} 
\label{app:related_work}

In this section we outline the details for the related work to our approach, including JKO-based methods in~\autoref{subsec:jko}, methods based on residual losses in~\autoref{subsec:normalizing_flow}, and action matching in~\autoref{subsec:am}. Further methods are discussed in \autoref{subsec:const_related_exten}-- \autoref{subsec:other-methods}. 


\subsection{JKO-based methods} \label{subsec:jko}
The dominant algorithmic framework to tackle our main goal is the Jordan--Kinderlehrer--Otto (JKO) scheme~\citep{jordan1998variational}. In this scheme we fix a discretization step $h>0$ and sequentially solve
\begin{equation}
\label{eq:JKO}
\tag{JKO}
\rho_{t+h} \;:=\; \argmin_{\nu\in \calP_2(\R^D)}\left[\calF(\nu)+\frac{1}{2h}W_2^2(\rho_t,\nu)\right].
\end{equation}
As $h\to 0$, this iteration scheme converges to the Wasserstein gradient flow of $\calF$ \cite[Ch.~11]{santambrogio2015optimal}. 
There are currently two strategies for the purpose of learning the potential $\calF$ from the data $\{q_t\}_{t\in \calT_{\mathrm{obs}}}$ using the JKO scheme, which we outline next. 

\paragraph{Forward simulation.} The JKOnet algorithm \citep{bunne2022proximal} iterates \eqref{eq:JKO}, starting from $\rho_0=q_0$, by solving a bi-level optimization. JKONet uses Brenier's theorem to characterize the solution to \eqref{eq:JKO} as an optimization problem over a convex function $\psi$. In particular, given a parametrized functional $\calF^{\xi}$, the inner optimization problem $\min_{\nu\in \calP_2(\R^D)}\left[\calF^{\xi}(\nu)+\frac{1}{2h}W_2^2(\rho_t,\nu)\right]$ is rewritten as,
\begin{equation}
\label{eq:JKO_equiv}
\begin{split}
&\min_{\nu\in \calP_2(\R^D)}\left[\calF^{\xi}(\nu)+\frac{1}{2h}W_2^2(\rho_t^{\xi},\nu)\right]\\
&=\min_{\nabla\psi^{\theta}: \psi^{\theta}:\R^D\to \R\textnormal{ convx}}\left[\calF^{\xi}(\nabla\psi_{\sharp}^{\theta}\rho_t^{\xi})+\frac{1}{2h}\int_{\R^D}\|\x-\nabla\psi^{\theta}(\x)\|^2\rho_t^{\xi}(\x)\,\mathrm{d}\x\right],
\end{split}
\end{equation}
where $\{\psi^{\theta}\}_{\theta}$ is taken to be a class of convex neural networks (ICNN); see~\citet{mokrov2021large} for using ICNNs for the proximal step. Then, given a minimizer $\theta$  set $\rho_{t+h}^{\xi}:=(\nabla\psi^{\theta})_{\sharp}\rho_t^{\xi}$, and then minimize over $\xi$, $W_2^2(\rho_{t+h}^{\xi},q_{t+h})$, (in fact a Sinkhorn divergence). The iJKOnet \citep{persiianov2025ijkonet} algorithm avoids the usage of ICNN by using a min-max formulation. 

\paragraph{First-order linearization.}
The JKOnet$^\star$ \citep{terpin2024learning} algorithms use the optimality condition of \eqref{eq:JKO},
\begin{equation}
\label{eq:JKO_FO}
\tag{JKO-FO}
\frac{\x'-\x}{h}\;=\;-\nabla \frac{\delta\calF}{\delta\rho_{t+h}}(\x') \qquad \forall (\x,\x')\in \supp \gamma_t,
\end{equation}
where $\gamma_t$ is the optimal transport plan between consecutive marginals in $\{q_t\}_{t\in \calT_{\mathrm{obs}}}$. In this style of algorithms the optimal transport problem is first solved between any two consecutive marginals in $\{q_t\}_{t\in \calT_{\mathrm{obs}}}$. Given these optimal transport plans $\{\gamma_t\}$ the functional $\calF$ is parametrized by a neural network $\calF^{\theta}$, and the optimization problems becomes minimizing over $\theta$ the residual of \eqref{eq:JKO_FO}. 

In both strategies the only learnable object is the functional $\calF$. The curve $\rho$ can be recovered only by iterating the scheme \eqref{eq:JKO}. Therefore, trajectory expressivity is coupled to $\calF$'s, and evaluating $\rho_t$ at unobserved times requires post-hoc simulation. Instead, our approach promotes the parametrization $\rho^\theta$ of the curve $\rho$ to a first-class learnable object on equal footing with $\calF^\theta$, coupled only through a residual loss. Both are optimized at the same level, and the curve's capacity is determined by its own parametrization, not by what the JKO operator can produce.




\subsection{Residual Losses}
\label{subsec:normalizing_flow}

Concurrent work of \citet{liu2026generative} considers the problem of solving a WGF given a known functional $\calF$ using a form of \eqref{eq:velocity_res}, resulting in \emph{GenWGP}. In particular, considering a $K$-step discretization of the interval $[0, T]$, GenWGP parameterizes the flow-map $\Phi^{\theta}(t, \x)$ as a normalizing flow, i.e., $\Phi^{\theta}(t_k, \cdot) = \Psi_k \circ \cdots \Psi_1 \circ \rho_0$, for invertible neural networks $\Psi_k$ ($k=1, \dots, K$). 

This flow is applied to a population of particles, $x_k^{(i)}$; densities can then be empirically evaluated via the normalizing flow, which enables an approximate evaluation of velocity fields using finite difference methods. The resulting approximation of \eqref{eq:velocity_res} is then optimized via gradient descent. 

The GenWGP approach is stated only for a known functional $\mathcal{F}$, but can be adapted to our setting where $\calF$ is unknown. The main drawback, compared to stitching, in this inverse problem setting is that GenWGP requires \emph{simulation}; in particular, the normalizing flow comprises solving a neural ODE. In contrast, stitching does not require an ODE solve, as the particles $\x^{\theta}$ are directly parameterized as part of the resulting optimization problem. Notably, however, GenWGP does not require a further approximation of the density $\rho_t$, e.g., in the form of a KDE.

\subsection{Action Matching}
\label{subsec:am}

\citet{neklyudov2023action} introduced action matching (AM) for learning the velocity field of a population dynamic, and \citet{neklyudov2024computational} extends the framework to a broader class of Wasserstein Lagrangian flows beyond pure gradient flows. The main idea of \citet{neklyudov2023action} is to fit a velocity field to data without access to ground-truth velocities. Adapted to our setting, the goal is to learn $\calF^\theta$ such that $-\nabla_\x\frac{\delta\calF^\theta}{\delta\rho_t^\theta}$ matches the (unknown) true velocity $\v_t^\star$ of the data. The naive objective
\begin{equation}
\label{eq:AM_ideal}
\int_0^T\!\!\int_{\R^D}\left\|\nabla_\x\dFdrtt(\x)+\v_t^\star(\x)\right\|^2q_t(\x)\,\mathrm{d}\x\,\mathrm{d}t
\end{equation}
cannot be optimized as written due to the unknown velocity. By proceeding as in \citet[Theorem 2.2]{neklyudov2023action}, using integration-by-parts, the $\theta$-dependent part of the objective is given by
\begin{align}
\label{eq:AM}
\calL_{\textnormal{AM}}(\theta) &= \int_{\R^D} \frac{\delta\calF^{\theta}}{\delta\rho_0^{\theta}}(\x)q_0(\x)\mathrm{d}\x-\int_{\R^D} \frac{\delta\calF^{\theta}}{\delta\rho_T^{\theta}}(\x)q_T(\x)\mathrm{d}\x \\
&\quad +\int_0^T\!\!\int_{\R^D}\!\Big[\frac{1}{2}\left\|\nabla_\x\dFdrtt(\x)\right\|^2+\partial_t\dFdrtt(\x)\Big]q_t(\x)\,\mathrm{d}\x\,\mathrm{d}t,
\end{align}
which can be evaluated using only samples from $q_t$. In our taxonomy, AM is an example of the velocity residual evaluated at data samples, with the integration-by-parts trick eliminating the need for $\v_t^\star$. Unlike the divergence-based methods, AM fuses residual and data fitting into a single objective. The price is that AM requires data at the temporal endpoints (for the boundary terms in \eqref{eq:AM}) and does not directly constrain $\rho_t^\theta$ at intermediate times---it learns $\calF^\theta$ such that its gradient flow has the right velocity, leaving the trajectory $\rho^\theta$ to be recovered post-hoc.

\subsection{Consistency by construction.} 
\label{subsec:const_related_exten}

An alternative to the residual-loss approach we take is to bake the continuity equation into the parametrization itself, so that any candidate $(\rho^\theta,\v^\theta)$ satisfies it by construction. Neural Conservation Laws \citet{richterpowell2022neural} parametrize divergence-free vector fields via differential forms, and \citet{hua2025simulation} managed simplify this construction. Their setting is closely related to ours: like stitching, it parametrizes the trajectory and the dynamics jointly, but it enforces the PDE via the architecture rather than via a residual loss, which trades off architectural flexibility against exact constraint satisfaction.

\subsection{Neural SDEs, Schr\"odinger bridges, and flow matching}
\label{subsec:nueral_sde_related_exten}

Neural SDEs \citep{li2020scalable,kidger2021efficient,kidger2021neural} are a popular way to learn dynamical systems: they parameterize the drift and diffusion terms of an SDE, trained by optimization of a variational objective. Contrary to our setting, however, neural SDEs are typically formulated and trained with respect to trajectory/pathwise data. Applying them to population data is non-trivial, accomplished via a Sinkhorn divergence in \citet{koshizuka2022neural}.

Schr\"odinger bridges \citep{leonard2014survey,debortoli2021diffusion,shi2023diffusion,koshizuka2022neural,chen2023deep}, diffusion models \citep{song2020score} and flow matching \citep{lipman2022flow,tong2023improving,albergo2022building,wang2025joint} learn velocity fields between distributions without imposing a gradient-flow structure. They are thus solving a strictly weaker problem: they recover dynamics, but not an underlying energy. Similarly to flow matching, normalizing flows parametrize a velocity network $\v_t^\theta$ and let $\rho_t^\theta$ be the continuous normalizing flow \citep{chen2018neural,papamakarios2021normalizing} obtained by transporting $\rho_0$ along $\v_t^\theta$. This provides exact normalized density evaluation, and permits learning the model parameters by maximizing the likelihood (i.e., minimizing the KL divergence) of the data at observed times under the model $\rho_t^\theta$. 

\subsection{Other methods}\label{subsec:other-methods}
We also mention several additional related methods from the literature which do not fall into the above categories. We leave detailed comparison of our method to those in a future work.

\citet{guan2024identifying,guan2025gradient} study identifiability of the drift and diffusion of an SDE from temporal marginals, complementary to our setting where the drift is constrained to be the gradient of a learned functional. \cite{carrillo2025} considers recovery of interaction kernel from gridded density data using regularized basis pursuit. In \cite{wei2026}, the authors derive a self-test loss based on the weak form of the stochastic evolution equation for the empirical measure.

\section{Wavy valley: details}
\label{app:wavy_valley}

This appendix collects all hyperparameters and protocol details for the wavy-valley experiment summarised in \autoref{sec:exp_wavy_valley} and \autoref{fig:wavy_valley}.

\paragraph{Potential and SDE.}
The wavy valley is the 2D landscape
\[
  V(x_1, x_2) \;=\; K\bigl(x_2 - \sin(\pi x_1 / 2)\bigr)^2 \;-\; \tau\, x_1,
  \qquad K = 0.6,\ \tau = 0.3,
\]
whose minimum-energy curve is the sinusoid $x_2 = \sin(\pi x_1 / 2)$. The first term confines particles to the valley floor; the second tilts the floor downward in $+x_1$. We simulate the SDE $\mathrm{d}x = -\nabla V(x)\,\mathrm{d}t + \sqrt{2\beta}\,\mathrm{d}W$ with $\beta = 0.00625$ from a tight Gaussian initial distribution centred at $(-3, 1)$ with std $0.1$, using Euler--Maruyama with at least $10$ substeps per integration interval.

\paragraph{Snapshot protocol.}
We observe the population at the irregular times $t \in \{0, 5, 10, 20, 30\}$. Each snapshot draws $N = 20$ particles from the same coupled simulation pool (matched particle identities across times), yielding $T = 5$ paired snapshots and $20$ training points per snapshot. The $t = 10 \to t = 20$ gap of width $10$ is intentional: a first-order JKO predictor cannot interpolate a curved trajectory across a step that long.

\paragraph{Stitching configuration.}
The stitching model uses $N=20$ particles with $K=50$ trajectory nodes spaced linearly between $t_0 = 0$ and $t_T = 30$. The KDE bandwidth is initialised at $0.5$ and trained; the entropy coefficient is initialised at $0.05$ and trained (matching lightspeed's setting); particle mixture weights are trainable softmax variables. The velocity residual uses an $\alpha = 0.5$ convex combination of centers quadrature and Nadaraya--Watson Monte Carlo (\autoref{app:stitching_derivation}), with one MC sample per particle per snapshot. The $V^\theta$ network is a $(64, 64)$ MLP. Training: Adam at $\text{lr} = 5 \cdot 10^{-3}$ for $10K$ iterations.

\paragraph{Lightspeed configuration.}
We use the published JKOnet$^\star_V$ recipe \citep{terpin2024learning} with one modification: the entropy coefficient is trainable rather than frozen at zero, matching stitching for fairness. Other hyperparameters: $(64, 64)$ hidden, $\text{lr} = 10^{-3}$, $10K$ iterations, OT-Hungarian coupling between consecutive snapshot pairs.

\paragraph{Metrics.}
We report (i) per-snapshot $W_2$ between trained-model samples and held-out particles; (ii) coefficient of determination $R^2(V)$ between learned and ground-truth potentials, evaluated on the data support (the values in \autoref{fig:wavy_valley}); (iii) pattern $R^2$ between gradient fields, the scale-invariant variant.


\section{Synthetic potential recovery: details}
\label{app:synthetic_full}

\paragraph{Metric definitions.}
We use the three metrics of \citet[Eqs.~17--20]{persiianov2025ijkonet}, all lower-is-better. Let $\rho^\theta_{t_{k+1}}$ denote the model's predicted marginal at observation time $t_{k+1}$ and $q_{t_k}$ the observed marginal at $t_k$.
\begin{itemize}\itemsep0.2em
\item \textbf{One-step EMD}: the Earth Mover's (1-Wasserstein) distance between predicted and observed marginal, averaged over consecutive transitions,
\[
\mathrm{EMD} \;:=\; \frac{1}{T-1}\sum_{k=0}^{T-2} W_1\bigl(\rho^\theta_{t_{k+1}},\, q_{t_{k+1}}\bigr).
\]
\item $L^2$\textbf{-UVP} (gradient-level error): the residual squared $L^2(q_{t_k})$ error in the learned gradient $\nabla V^\theta$ as a fraction of the variance of the true gradient,
\[
L^2\text{-}\mathrm{UVP} \;:=\; \frac{\mathbb{E}_{\x\sim q_{t_k}}\bigl\|\nabla V^\theta(\x) - \nabla V(\x)\bigr\|^2}{\mathrm{Var}_{\x\sim q_{t_k}}\bigl[\nabla V(\x)\bigr]} \;\times\; 100\%.
\]
$L^2$-UVP measures recovery of the gradient field independently of additive constants in $V$, which is the right quantity for a Wasserstein gradient flow (the dynamics is invariant to such constants).
\item $\mathrm{Bd}^2_{W_2}$\textbf{-UVP} (Bures--Wasserstein UVP): the squared Bures--Wasserstein distance between Gaussian approximations of the predicted and observed marginals, normalized by the trace of the observed covariance. With $\mu^\theta, \Sigma^\theta$ the mean and covariance of $\rho^\theta_{t_{k+1}}$ and $\mu, \Sigma$ those of $q_{t_{k+1}}$,
\[
\mathrm{Bd}^2_{W_2}\text{-}\mathrm{UVP} \;:=\; \frac{\|\mu^\theta - \mu\|^2 \;+\; \mathrm{tr}\!\bigl(\Sigma^\theta + \Sigma - 2(\Sigma^{1/2}\Sigma^\theta\,\Sigma^{1/2})^{1/2}\bigr)}{\mathrm{tr}(\Sigma)} \;\times\; 100\%.
\]
This captures how well the predicted first two moments match the observed ones.
\end{itemize}

\paragraph{Pattern and raw $R^2$ for $V$.}
The galleries (\autoref{fig:synthetic_compact}, \autoref{fig:synthetic_gallery_paired}, \autoref{fig:synthetic_gallery_unpaired}) report two coefficients of determination between the learned and ground-truth potentials, both evaluated on a uniform grid covering the data support. $R^2_{\mathrm{raw}}$ is the standard coefficient of determination,
\[
R^2_{\mathrm{raw}}(V^\theta, V) \;:=\; 1 - \frac{\sum_i \bigl(V^\theta(\x_i) - V(\x_i)\bigr)^2}{\sum_i \bigl(V(\x_i) - \bar V\bigr)^2},
\]
where $\bar V$ is the mean of $V$ over the grid. $R^2_{\mathrm{raw}}$ is sensitive to additive and multiplicative shifts of $V^\theta$. $R^2_{\mathrm{pattern}}$ replaces the residual-sum-of-squares ratio with the squared Pearson correlation,
\[
R^2_{\mathrm{pattern}}(V^\theta, V) \;:=\; \mathrm{corr}\bigl(V^\theta, V\bigr)^2,
\]
again evaluated on the same grid; this is scale- and shift-invariant. Because Wasserstein gradient flow is invariant to additive constants in $V$ and identifiable only up to a scale by $\beta$ \citep[App.~A]{persiianov2025ijkonet}, $R^2_{\mathrm{pattern}}$ is the natural similarity score for the recovered potential, while $R^2_{\mathrm{raw}}$ additionally penalises any residual scale or offset.

\paragraph{Stitching configuration.}
\begin{itemize}\itemsep0pt
\item \textbf{Trajectory:} $N=1{,}000$ particles per run, length $K=50$, identity-coupled at initialization.
\item \textbf{Functional:} $\calF^\theta[\rho_t^\theta] = c_V\,\mathbb{E}_{\rho_t^\theta}[V^\theta(\x)]$ with $V^\theta$ a $(64,64)$ MLP and $c_V > 0$ a softplus-parametrised scalar (entropy and interaction terms disabled --- the \citet{terpin2024learning} benchmark is deterministic gradient flow of a single potential).
\item \textbf{Density model:} per-dimension Gaussian KDE bandwidth set by Silverman's rule, frozen uniform mixture weights $w_k = 1/N$.
\item \textbf{Loss:} velocity residual evaluated at particle centers ($\alpha = 0$, no KDE--MC perturbation), trapezoidal time scheme, plus a KDE-KL data divergence at the five observed snapshots.
\item \textbf{Optimizer:} full-batch Adam at learning rate $5\cdot 10^{-3}$ with cosine decay, $2{,}000$ steps.
\item \textbf{Configuration is identical across all $30$ runs} ($15$ potentials $\times$ $\{$paired, unpaired$\}$); each run takes a few minutes on a single CPU.
\end{itemize}

\paragraph{Baseline configuration.} JKOnet$^\star_V$ and iJKOnet$_V$ are run directly from the upstream codebases of \citet{terpin2024learning} and \citet{persiianov2025ijkonet} respectively, both at default hyperparameters: a $(64,64)$ MLP $V^\theta$ matching stitching, no entropy or interaction terms, single seed, on the same train/test splits as stitching. JKOnet$^\star_V$ uses Hungarian-OT couplings between consecutive snapshots; iJKOnet$_V$ uses its inverse-JKO solver with $K=5$. We train JKOnet$^\star_V$ for $100$ epochs and iJKOnet$_V$ for $2{,}000$ epochs (the upstream defaults). Metrics are pulled from each method's own evaluation pipeline (parsed from upstream stdout for iJKOnet$_V$, saved-params $+$ a forward SDE rollout in our pipeline for JKOnet$^\star_V$); iJKOnet$_V$'s CLI does not expose $L^2$-UVP, so those cells are marked `--' in \autoref{tab:synthetic2d_abs}. Because of compute, the baseline runs are restricted to the $6$ potentials \citet[Tab.~3]{persiianov2025ijkonet} flag as most paired$\to$unpaired sensitive; stitching's appendix numbers below cover all $15$.

\paragraph{Per-potential numbers.}
\autoref{tab:synthetic2d} reports stitching's three metrics on each of the $15$ landscapes of \citet{terpin2024learning}, in both regimes. The main-text \autoref{tab:synthetic2d_abs} compares stitching against the JKO baselines on the $6$ sensitive potentials. The consistency claim of the main text is visible row-by-row in \autoref{tab:synthetic2d}: paired and unpaired columns differ by less than $\sim$2$\times$ on every metric and every potential except the four most delicately structured (\texttt{sphere}, \texttt{bohachevsky}, \texttt{rotational}, and to a lesser extent \texttt{relu}), where stitching's $L^2$-UVP and $\mathrm{Bd}^2_{W_2}$-UVP increase but never collapse.

\begin{table}[tb]
\centering
\caption{Stitching's per-potential metrics on the $15$ two-dimensional landscapes of \citet{terpin2024learning}, in the \emph{paired} (correlated trajectories, original protocol) and \emph{unpaired} (independent snapshots, \citet{persiianov2025ijkonet} protocol) regimes. Same configuration across all $30$ runs ($1{,}000$ particles, $2{,}000$ steps, single seed). Metrics follow \citet{persiianov2025ijkonet}: one-step EMD, $L^2$-UVP (gradient-level error), and $\mathrm{Bd}^2_{W_2}$-UVP (Bures--Wasserstein UVP); $L^2$-UVP and $\mathrm{Bd}^2_{W_2}$-UVP in percent. Definitions in the metric paragraph above. Lower is better.}
\label{tab:synthetic2d}
\small
\setlength{\tabcolsep}{4pt}
\begin{tabular}{llrrrrrr}
\toprule
 & & \multicolumn{3}{c}{paired (correlated)} & \multicolumn{3}{c}{unpaired (independent)} \\
\cmidrule(lr){3-5} \cmidrule(lr){6-8}
\# & potential & EMD\,$\downarrow$ & $L^2$-UVP\,$\downarrow$ & $\mathrm{Bd}^2_{W_2}$-UVP\,$\downarrow$ & EMD\,$\downarrow$ & $L^2$-UVP\,$\downarrow$ & $\mathrm{Bd}^2_{W_2}$-UVP\,$\downarrow$ \\
\midrule
1  & flowers           & 0.30 & 0.00 & 0.18  & 0.31 & 0.01 & 0.24 \\
2  & styblinski\_tang  & 0.26 & 0.02 & 0.15  & 0.29 & 0.04 & 0.29 \\
3  & holder\_table     & 0.29 & 0.06 & 0.18  & 0.30 & 0.06 & 0.24 \\
4  & zigzag\_ridge     & 0.29 & 0.03 & 0.17  & 0.30 & 0.04 & 0.29 \\
5  & oakley\_ohagan    & 0.24 & 0.01 & 0.20  & 0.25 & 0.02 & 0.28 \\
6  & watershed         & 0.30 & 0.00 & 0.19  & 0.30 & 0.01 & 0.24 \\
7  & ishigami          & 0.30 & 0.01 & 0.17  & 0.30 & 0.01 & 0.25 \\
8  & friedman          & 0.29 & 0.06 & 0.20  & 0.31 & 0.06 & 0.27 \\
9  & sphere            & 0.62 & 0.86 & 1.45  & 0.64 & 0.91 & 1.24 \\
10 & bohachevsky       & 0.39 & 7.39 & 13.99 & 0.38 & 7.46 & 12.93 \\
11 & wavy\_plateau     & 0.27 & 0.02 & 0.51  & 0.29 & 0.04 & 0.54 \\
12 & double\_exp       & 0.34 & 0.02 & 0.22  & 0.35 & 0.06 & 0.29 \\
13 & relu              & 0.35 & 0.04 & 0.18  & 0.38 & 0.10 & 0.33 \\
14 & rotational        & 0.41 & 0.82 & 1.60  & 0.45 & 0.82 & 1.80 \\
15 & flat              & 0.30 & 0.00 & 0.19  & 0.31 & 0.01 & 0.24 \\
\bottomrule
\end{tabular}
\end{table}

\begin{figure}[h]
\centering
\includegraphics[width=\linewidth]{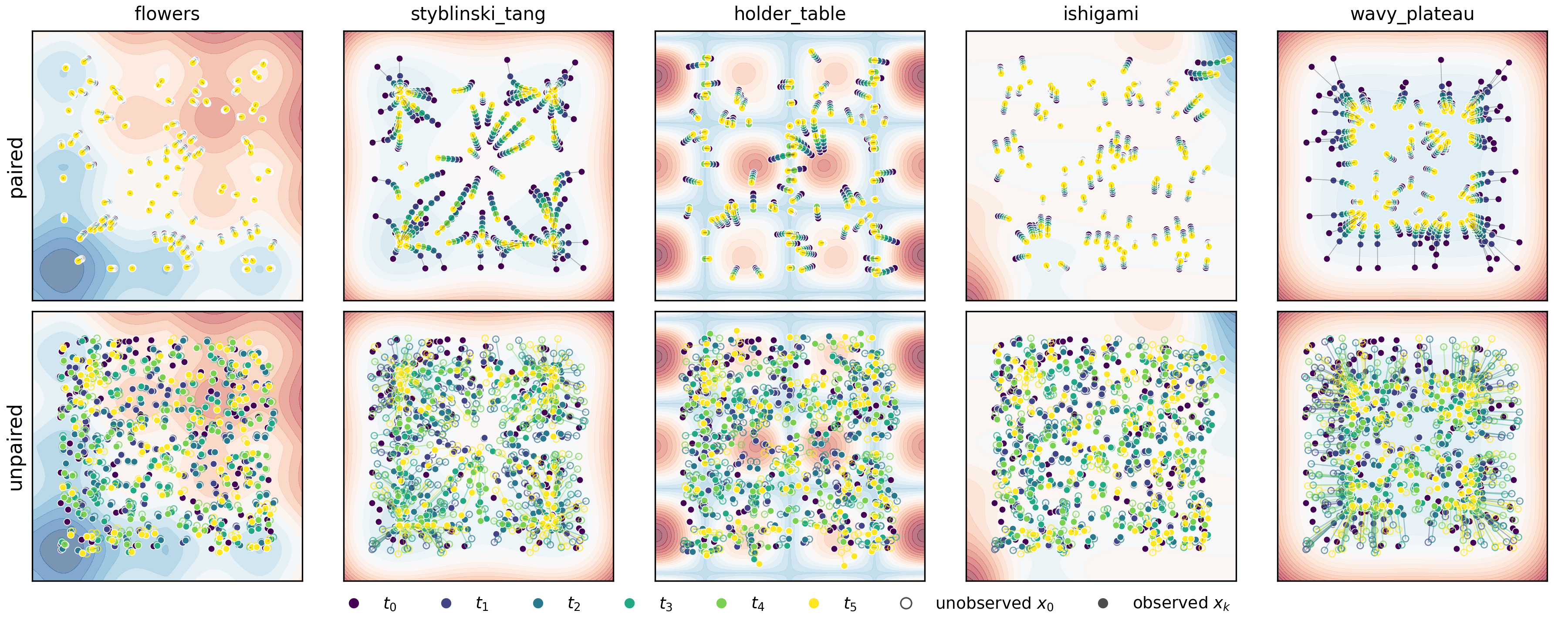}
\caption{\textbf{Paired vs.\ unpaired snapshot evolution on the synthetic  potentials following \citet{persiianov2025ijkonet}.} Top: In `paired' setting we observe same particles driven by a potential gradient over time. Bottom: In `unpaired` trajectory structure is lost, as if the observations remove the corresponding particles from the system. Both cases have the same number of observations.}
\label{fig:flowers_paired_vs_unpaired}
\end{figure}

\begin{figure}[tb]
\centering
\includegraphics[width=\linewidth]{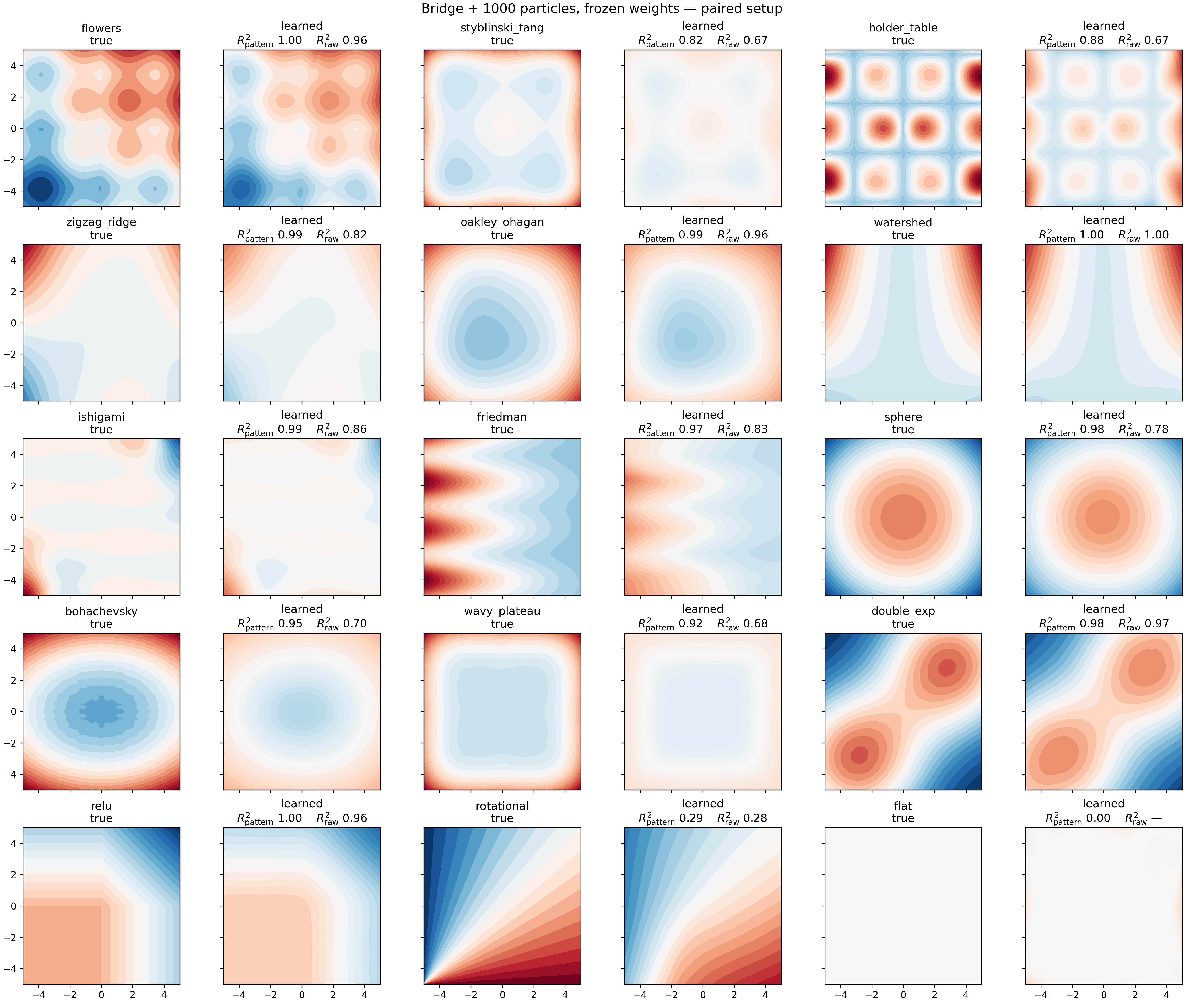}
\caption{\textbf{Stitching recovers the level-set geometry of $V$ on $11$ of $14$ informative landscapes.} Synthetic potential recovery, \emph{paired} regime: true $V$ (left) vs learned $V^\theta$ (right) for each potential, mean-centered with shared per-potential color scale; titles report scale-invariant $R^2_{\mathrm{pattern}}$ and on-support $R^2_{\mathrm{raw}}$.}
\label{fig:synthetic_gallery_paired}
\end{figure}

\begin{figure}[tb]
\centering
\includegraphics[width=\linewidth]{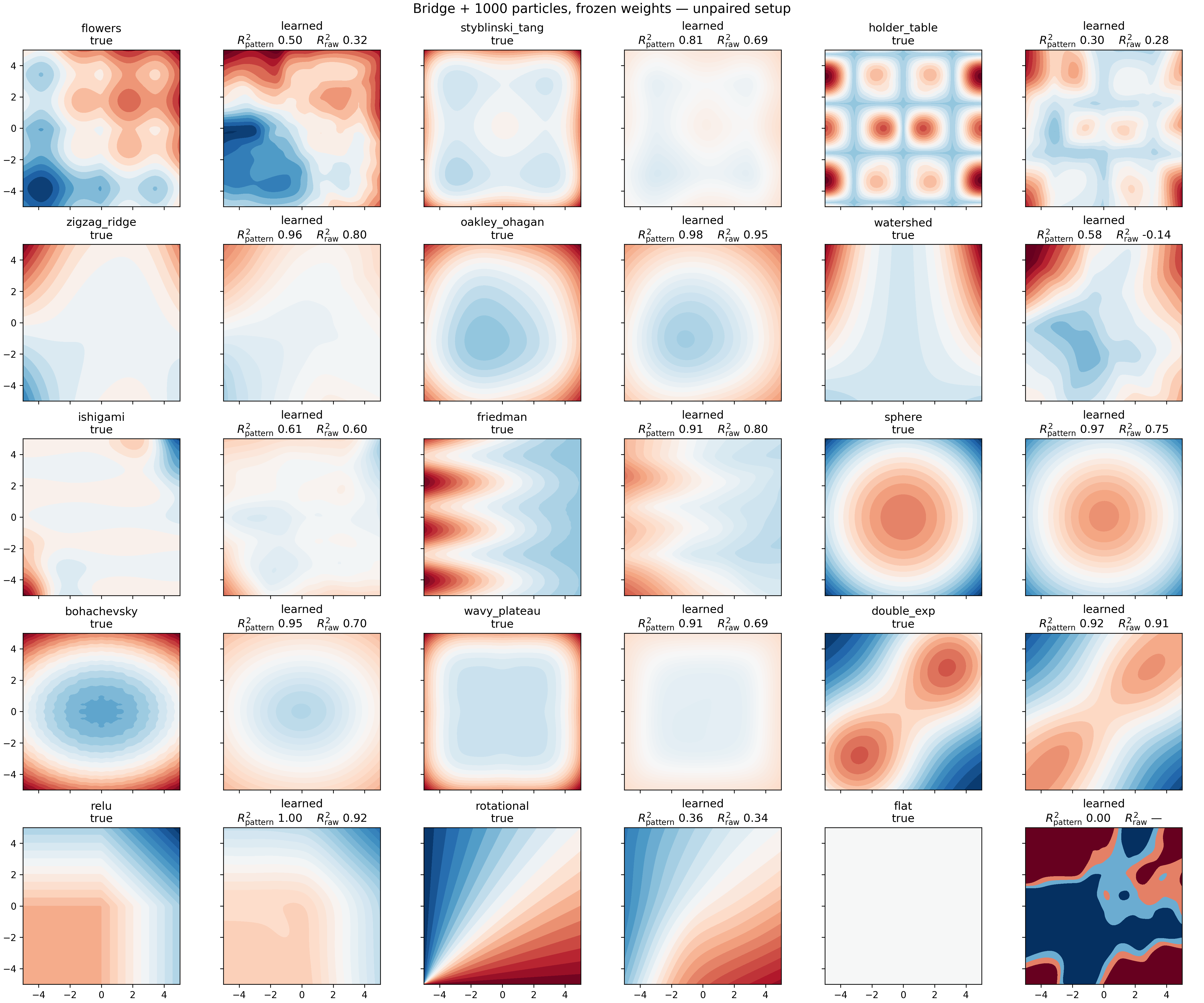}
\caption{\textbf{Stitching remains robust to independently-sampled snapshots; detectable degradation only on the four most delicately structured potentials.} Synthetic potential recovery, \emph{unpaired} regime; same conventions as \autoref{fig:synthetic_gallery_paired}.}
\label{fig:synthetic_gallery_unpaired}
\end{figure}

\FloatBarrier

\section{Single-cell trajectory inference: details}
\label{app:rna_details}

\paragraph{Dataset and preprocessing.}
The embryoid body (EB) dataset of \citet{moon2019visualizing} comprises $\sim 17{,}000$ cells observed across five 3-day windows of human embryonic stem cell differentiation. We use the PCA-reduced version distributed by \citet{tong2020trajectorynet}: each cell is represented by its first $5$ principal components, and time labels are scaled to $[0,1]$. We use a 70/30 particle-level train/test split, deterministic per seed.

\paragraph{Stitching architecture.}
\begin{itemize}\itemsep0pt
\item \textbf{Static} $V^\theta(\x)$: MLP, input dim $5$, hidden $(64,64)$, SiLU activations with residual skips, scalar output.
\item \textbf{Time-varying} $V^\theta(\x,t)$: same MLP but input dim $6$, with $t$ concatenated to $\x$. Matches the design of iJKOnet$_{t,V}$ in \citet[Sec.~5]{persiianov2025ijkonet} and JKOnet$^*_{t,V}$ in \citet{terpin2024learning}.
\item \textbf{Trajectory:} $K=50$ learnable snapshots between $t_{\min}$ and $t_{\max}$, $N=100$ particles per snapshot, OT-coupled at initialization (Hungarian assignment + linear interpolation between consecutive observed marginals).
\item \textbf{Density model:} per-dimension Gaussian KDE bandwidth (learnable, softplus-parametrized) and learnable mixture weights (softmax-normalized).
\item \textbf{Functional:} $\calF^\theta[\rho_t^\theta] = c_V\,\mathbb{E}_{\rho_t^\theta}[V^\theta(\x)] + c_H\,\mathbb{E}_{\rho_t^\theta}[\log\rho_t^\theta]$ with $c_V,c_H>0$ softplus-parametrized scalars.
\end{itemize}

\paragraph{Training.}
\begin{itemize}\itemsep0pt
\item \textbf{Loss:} 
 $\calL = w_{\mathrm{kl}}\,\calD_{\mathrm{KL}}(\rho^\theta_{t_{\mathrm{obs}}}, q_{t_{\mathrm{obs}}}) 
+ w_{\mathrm{vel}}\,\calR_{\mathrm{vel}}[\calF^\theta,(\v^{\theta},\rho^\theta)]$ with the trapezoidal-scheme velocity residual (\autoref{eq:stitch_ideal_global_approx}), $\alpha=0$ (no KDE-MC perturbation), kinetic factor $1$.
\item \textbf{Optimizer:} Adam, batch size $256$, learning rate $5 \cdot 10^{-3}$ with cosine decay to $5\%$ of the initial value.
\item \textbf{Epochs:} $10{,}000$.
\item \textbf{Seeds:} $0,\ldots,4$ for the leave-two-out comparison; a single seed for the full-data benchmark and \autoref{fig:rna_panels}.
\item \textbf{Wall time:} $\sim 2$ minutes per run on a CPU laptop (M-series).
\end{itemize}

\paragraph{Evaluation protocols.}
\begin{itemize}\itemsep0pt
\item \textbf{Full data} (\autoref{tab:rna_full}): train on all five marginals; report $W_1$ between the stitching KDE marginal at each observed $t$ and the held-out test cells at that $t$. Forward-rollout from the learned trajectory; no JKO chain.
\item \textbf{Leave-two-out} (\autoref{tab:rna_interp}): restrict training to $t\in\{0,2,4\}$; evaluate at held-out $t\in\{1,3\}$ via KDE marginals of the learned particle cloud at the target time. No retraining or architectural change; the model is continuous-time by construction. Metric $W_2$ to match \citet[Table~1]{persiianov2025ijkonet}.
\item All distances computed with the POT library: EMD via \texttt{emd2(Euclidean)} and $W_2$ via $\sqrt{\texttt{emd2}(\text{sqEuclidean})}$.
\end{itemize}

\paragraph{Leave-two-out results.}
\autoref{tab:rna_interp} compares stitching against the full set of baselines reported in \citet[Table~1]{persiianov2025ijkonet}, all using the same $W_2$ metric. Stitching with a static $V$ already outperforms every published baseline (mean $W_2$ $0.92$); the time-varying $V^\theta(\x,t)$ variant further reduces the mean to $\mathbf{0.88}$, with the best published competitor (iJKOnet$_{t,V}$) at $0.92$.

\begin{table}[tb]
\centering
\caption{Leave-two-out temporal interpolation on the EB single-cell dataset (5\,D), $W_2$ ($\downarrow$). Models trained on $t\in\{0,2,4\}$; evaluated at held-out $t=1$ and $t=3$. Baseline results from \citet[Table~1]{persiianov2025ijkonet}; our results averaged over 5 seeds.}
\label{tab:rna_interp}
\adjustbox{max width=\linewidth}{%
\begin{tabular}{lllll}
\toprule
Method      & $t=1$ & $t=3$ & Mean & Citation \\
\midrule
TrajectoryNet  & $2.03$\stderr{0.04} & $1.93$\stderr{0.08} & $1.98$ & \citet{tong2020trajectorynet} \\
Vanilla-SB     & $1.49$\stderr{0.06} & $1.55$\stderr{0.03} & $1.52$ & \citet{vargas2021solving} \\
MMSB           & $1.27$\stderr{0.03} & $1.57$\stderr{0.05} & $1.42$ & \citet{shen2025mmsb} \\
DMSB           & $1.13$\stderr{0.08} & $1.45$\stderr{0.16} & $1.29$ & \citet{chen2023deep} \\
\midrule
\multicolumn{5}{l}{\textit{Static potential}} \\
JKOnet$^*_V$   & $1.15$\stderr{0.03} & $2.53$\stderr{0.01} & $1.84$ & \citet{terpin2024learning} \\
iJKOnet$_V$    & $1.08$\stderr{0.01} & $1.15$\stderr{0.00} & $1.12$ & \citet{persiianov2025ijkonet} \\
Stitching, $V^\theta(\x)$ & $\mathbf{0.85}$\stderr{0.06} & $\mathbf{0.99}$\stderr{0.03} & $\mathbf{0.92}$\stderr{0.03} & This paper \\
\midrule
\multicolumn{5}{l}{\textit{Time-varying potential}} \\
JKOnet$^*_{t,V}$    & $4.41$\stderr{1.50} & $2.77$\stderr{0.20} & $3.59$ & \citet{terpin2024learning} \\
iJKOnet$_{t,V}$     & $0.98$\stderr{0.04} & $\mathbf{0.85}$\stderr{0.02} & $0.92$ & \citet{persiianov2025ijkonet} \\
Stitching, $V^\theta(\x,t)$ & $\mathbf{0.84}$\stderr{0.05} & $0.92$\stderr{0.03} & $\mathbf{0.88}$\stderr{0.04} & This paper \\
\bottomrule
\end{tabular}%
}
\end{table}

\section{Interaction dynamics: details}
\label{app:cis}

This appendix collects the full setup, model, training, and metric definitions deferred from \autoref{sec:exp_interaction}.

\paragraph{Dataset.}
The interaction dynamics data is a single-trajectory simulation of a 2D SDE \eqref{eq:mckean-vlasov-sde}
\[
  \mathrm{d}X_t^i = -\nabla_{\x} V(X_t^i)\,\mathrm{d}t - \frac{1}{N}\!\sum_{j\neq i}\nabla_{x} W(X_t^i{-}X_t^j)\,\mathrm{d}t + \sigma\,\mathrm{d}B_t^i, \quad, i=1,\dots, N,
\]
with $V(x){=}\alpha(\|x\|^2{-}\beta)^2$ ($\alpha{=}0.1, \beta{=}4$, minimum on the ring $\|x\|{=}2$), $W(r){=}\eta\,e^{-r^2}$ ($\eta{=}{-}2.2$), $\sigma{=}0.045$, $N{=}310$ particles, integrated with Euler--Maruyama at $\mathrm{d}t{=}0.05$ over $[0,200]$ ($T{=}201$ snapshots) from a Gaussian initial condition. We use a $50/50$ particle-level train/test split applied uniformly across all $T$ snapshots.

\paragraph{Energy parametrization.}
We learn $\calF^\theta[\rho_t^\theta]{=}c_V\,\E_{\rho_t^\theta}[V^\theta(\x)]+c_W\,\E_{\rho_t^\theta\otimes\rho_t^\theta}[W^\theta(\|\x{-}\x'\|)]+c_H\,\E_{\rho_t^\theta}[\log\rho_t^\theta]$ with $V^\theta:\R^2{\to}\R$ and the radial kernel $W^\theta:\R_{\geq 0}{\to}\R$, both parametrized as two-layer $(64,64)$ MLPs (zero-init last layer), and softplus-parametrized positive coefficients $c_V,c_W,c_H>0$. The radial parametrization $W^\theta(\|\x{-}\x'\|)$ enforces translational and rotational symmetry of the kernel without restricting its functional form.

\paragraph{Stitching model.}
The trajectory $\rho^\theta$ is represented by $100$ learnable particles at $K{=}50$ snapshots between $t_0{=}0$ and $t_{200}$, OT-coupled to the data at initialization, with a Gaussian KDE marginal. The KDE bandwidth is initialized by Silverman's rule and learned. Mixture weights are held uniform ($w_k{=}1/N$).

\paragraph{Training.}
Full-batch Adam at $\text{lr}{=}5{\cdot}10^{-3}$ for $1{,}000$ epochs, optimizing the centers-quadrature velocity residual~\eqref{eq:velocity_res} (trapezoidal scheme, $\alpha{=}0$ in the convex combination) plus a KDE KL data-fit term, identical to the synthetic recovery setup of \autoref{sec:exp_synthetic}. CPU wall time on a 2023 MacBook is ${\approx}4$ minutes per run.

\paragraph{Baseline.}
We compare against JKOnet$^\star$ \citep{terpin2024learning} on the same train/test split. 

\paragraph{Pattern $R^2$ for $V$ and $W$.}
WGF inference from snapshots is identifiable only up to a joint $(V,W,\sigma)$-scale ambiguity \citep[App.~A]{persiianov2025ijkonet}: only $\nabla V/\sigma$ and $\nabla W/\sigma$ are determined by the data, not the absolute scale of $V,W$. We therefore report the scale- and shift-invariant pattern $R^2$, equal to the squared Pearson correlation between learned and true fields on the relevant support: a uniform grid covering the data support for $V$, and the band $r\in[r_{5\%},r_{95\%}]$ of observed pair distances at $t{=}100$ for $W$. This is the same metric used in \autoref{tab:synthetic2d}.

\section{Beyond gradient flows: non-conservative interactions}
\label{app:nongradient}
The residual framework of \autoref{sec:residuals} only requires a target velocity field; it does not require that target to be the Wasserstein gradient of any functional. This is relevant in  applications, as many interacting particle systems of practical interest have non-conservative pairwise interactions, including chiral active matter \citep{liebchen2022} and the broader class of non-reciprocal collective systems \citep{fruchart2021}, none of which arise as Wasserstein gradient flows. Concretely, for any model curve $\rho^\theta$ with Eulerian velocity $\v^\theta$ and any model velocity field $\u^\theta$,
\begin{equation}
\label{eq:generic_velocity_residual}
\calR_{\textnormal{vel}}[\u^\theta, (\rho^\theta, \v^\theta)]
\;:=\;\int_0^T\!\!\int_{\R^D}\bigl\|\v_t^\theta(\x)-\u_t^\theta(\x)\bigr\|^2\,\rho_t^{\theta}(\x)\,\mathrm{d}\x\,\mathrm{d}t
\end{equation}
is a well-defined nonnegative loss whose minimum is the curve whose velocity matches $\u^\theta$. Choosing $\u_t^\theta=-\nabla_\x\dFdrtt$ recovers \eqref{eq:velocity_res} and the WGF stitching of \autoref{sec:vel_res}; any other choice of $\u^\theta$ instantiates stitching for a different class of dynamics. In particular, the KDE parametrization \eqref{eq:particle_param}, the centers approximation \eqref{eq:kde_approx}, the time discretization leading to \eqref{eq:stitch_ideal_global_approx}, and the data divergence $\calD(\rho_t^\theta,q_t)$ are all independent of any gradient structure of $\u^\theta$.

Here we illustrate the extension on a non-conservative interacting particle system. We replace the symmetric scalar interaction $\inter$ in \eqref{eq:functional} by a vector-valued pairwise kernel $\mathbf{K}:\R^D\to\R^D$. The corresponding SDE  reads
\begin{equation}
\label{eq:mckean-vlasov-vector-kernel}
\mathrm{d}X_t^i \;=\; -\nabla \pot(X_t^i)\,\mathrm{d}t \;+\; \frac{1}{N}\sum_{j\neq i}\mathbf{K}(X_t^i-X_t^j)\,\mathrm{d}t \;+\; \sigma\,\mathrm{d}B_t^i,
\end{equation}
in direct analogy with the conservative SDE of \autoref{app:cis}. When $\mathbf{K} = -\nabla \inter$ for a scalar $\inter$, \eqref{eq:mckean-vlasov-vector-kernel} is a Wasserstein gradient flow of \eqref{eq:functional}; when $\mathbf{K}$ has nonzero curl, no scalar $\inter$ satisfies $\mathbf{K}=-\nabla \inter$, and \eqref{eq:mckean-vlasov-vector-kernel} cannot be written as a WGF. A canonical non-conservative case in two dimensions is the \emph{chiral kernel}
\begin{equation}
\label{eq:chiral_kernel_def}
\mathbf{K}(\x'-\x) \;=\; \alpha\,\bigl(-\nabla \inter(\x'-\x)\bigr) \;+\; \omega\,R_{\pi/2}\bigl(-\nabla \inter(\x'-\x)\bigr),
\qquad R_{\pi/2}\!=\!\begin{pmatrix}0 & -1\\ 1 & \phantom{-}0\end{pmatrix},
\end{equation}
where $\alpha,\omega\in\R$ are scalars and $\inter:\R^2\to\R$ is a radial scalar potential. The first summand is a standard conservative attraction/repulsion; the second is a 90$^\circ$ rotation of the same gradient, inducing a circulating component. Whenever $\omega\neq 0$ the kernel has nonzero curl.

\paragraph{Parametrization.}
We parametrize the learned kernel as an unrestricted vector-valued MLP
\begin{equation}
\label{eq:chiral_kernel_param}
\mathbf{K}^\theta:\R^2\to\R^2,
\end{equation}
that ingests the displacement $\mathbf{r}{=}\x'{-}\x$ directly and outputs a 2-vector, with no structural prior --- in particular, no built-in decomposition into radial and chiral parts. We omit both the confinement $\pot^\theta$ and the entropy term in this experiment, so the target velocity field reduces to the convolution of $\mathbf{K}^\theta$ against the curve,
\begin{equation}
\label{eq:chiral_target_velocity}
\u_t^\theta(\x) \;=\; (\mathbf{K}^\theta\ast\rho_t^\theta)(\x), \qquad (\mathbf{K}^\theta\ast\rho_t^\theta)(\x)=\int_{\R^D}\mathbf{K}^\theta(\x'-\x)\,\rho_t^\theta(\x')\,\mathrm{d}\x'.
\end{equation}
The stitching loss \eqref{eq:stitch_ideal_global_approx} is structurally unchanged: substitute $-\u_t^\theta$ for $\nabla_\x\dFdrtt$ and approximate $\mathbf{K}^\theta\ast\rho_t^\theta$ at the particle centers analogously to \eqref{eq:approx_inter_kde},
\begin{equation}
\label{eq:approx_chiral_kernel}
(\mathbf{K}^\theta\ast\rho_t^\theta)(\x_{t_j,k}^\theta)\;\approx\;\sum_{l=1}^{N}\wt_l\,\mathbf{K}^\theta(\x_{t_j,l}^\theta-\x_{t_j,k}^\theta).
\end{equation}

\paragraph{Setup.}
We simulate \eqref{eq:mckean-vlasov-vector-kernel} on $[0,T]$ with $T=800$, $\pot\equiv 0$, $\sigma=0$, $N=500$ particles, and the chiral kernel \eqref{eq:chiral_kernel_def}, where $\inter$ is a generalised Morse potential \citep{dorsogna2006}
\[
\inter(r) \;=\; C_r\,e^{-r/\ell_r} - C_a\,e^{-r/\ell_a},
\qquad (C_r,\ell_r,C_a,\ell_a)=(1.0,\,0.5,\,0.375,\,1.5),
\]
with chirality $\omega=1.5$ and radial scale $\alpha=0.2$. The initial condition is a mixture of two horizontal Gaussian blobs centred at $(0,\pm 2)$ with covariance $\mathrm{diag}(1.5^2,0.2^2)$; with $\omega\neq 0$ the two lumps orbit each other while the kernel's radial component would otherwise relax them to the rotationally-symmetric Morse equilibrium. Particles are saved at $\Delta t = 0.5$ via Euler--Maruyama with internal step $0.05$, yielding $1601$ marginal snapshots.

\paragraph{Stitching configuration.}
$\mathbf{K}^\theta$ is a $(64,64)$ MLP with SiLU activations and a 2-D output; the final-layer weights are rescaled by $10^{-2}$ at initialization so that $\mathbf{K}^\theta\!\approx\!0$ at start while gradients still propagate through every layer. We use $200$  particles. The confinement and entropy terms are not part of the model. The learning setup otherwise follows \autoref{app:cis}.

\paragraph{Result.}
\autoref{fig:chiral_snapshots} in the main text shows the stitching marginals tracking the rotation of the two-lump pattern across snapshots, with the learned curve $\rho^\theta$ following the data $q$ at the displayed times.


\end{document}